\documentclass{article}
\usepackage{iclr2027_conference,times}
\iclrfinalcopy
\newcommand{\coderepo}{https://github.com/urbanke/learning-source-structure}

\usepackage{amsmath,amssymb,amsthm}
\usepackage{graphicx}
\usepackage[section]{placeins}
\usepackage{needspace}
\usepackage{booktabs}
\usepackage{xcolor}
\usepackage{microtype}
\usepackage{xurl}
\usepackage[colorlinks=true,linkcolor=blue,citecolor=blue,urlcolor=blue]{hyperref}

\usepackage{xspace}

\newtheorem{lemma}{Lemma}

\theoremstyle{definition}

\theoremstyle{remark}

\title{Learning Source Structure:\\ Bayesian Prediction with Transformers}

\author{Meir Feder \\
School of Electrical and Computer Engineering \\
Tel Aviv University, Tel Aviv, Israel \\
\texttt{meir@tau.ac.il}
\AND
Ruediger Urbanke \\
School of Computer and Communication Sciences \\
EPFL, Lausanne, Switzerland \\
\texttt{rudiger.urbanke@epfl.ch}
}

\begin{document}
\maketitle
\fancyhead[L]{Preprint}

\begin{abstract}
A transformer trained on sequences whose source parameters are drawn afresh each time can learn only what the sequences share: the source family and its prior. The Bayesian predictor that knows both, which we call informed Bayes, is the unique minimizer of the training loss and is computable from the family alone; we use it as a quantitative, falsifiable reference. Where earlier work compared transformers with Bayes on sources whose contexts can be learned independently, we build Markov sources with up to 64 symbols in which observations in one context inform prediction in another: product sources with known or hidden encodings, transition rows sharing an unknown base law and depth, and a mixture of three families. Informed Bayes is evaluated numerically from the row integrals of the layered simplex architecture. Our main result is that informed Bayes predicts, from the family and prior alone, how a trained transformer behaves in four tests that go beyond the learning curve: the loss that remains under a source outside the training family (predicted 2.237 bits, measured 2.309), the doubling of the extra loss when two independent components change rather than one (0.345 predicted, 0.350 measured), the benefit of other contexts when a token appears for the first time, and the depth at which copying an earlier continuation starts to pay. On the learning curves themselves, small transformers match informed Bayes when the shared structure is visible in the tokens, and hidden structure needs more layers and training, with gaps below 0.01 bits per token for the largest models. We explain why the imitation is possible and where it must fail: informed Bayes decomposes into sufficient statistics that attention selects and accumulates and an estimator, fixed by the prior, that the feed-forward network learns; in the architecture-as-prior framework, training pays the complexity of informed Bayes under the architecture's prior, and finite-memory universal prediction bounds how closely a bounded network can reach it.
\end{abstract}

\section{Introduction}
\label{sec:intro}

A transformer that is trained on many sequences can learn what their sources
have in common. Suppose every sequence comes from a source with freshly drawn
parameters. Then training cannot reveal the parameters of the next sequence.
It can reveal the family of sources, and how the parameters within one source
are related: each token may be a pair whose two parts follow independent
chains, or all contexts may share a preference for certain tokens. A predictor that knows such a relation can pool observations across contexts. How well does a transformer learn to use this shared structure?

There is a natural yardstick. The \emph{informed Bayesian predictor},
informed Bayes for short, is told the source family and the prior. It is not
told the parameters of the sequence at hand; after each token it updates its posterior
belief about them and averages the next-token prediction over this belief.
When sequences are drawn from the family and the prior, no predictor that
sees the same tokens has a smaller expected log loss. The transformer is
trained to minimize this same loss and has the memory to approximate the informed Bayes. We ask two questions. How close does a
transformer, which must learn the family from examples, get to
informed Bayes, even in shared structure? And can this correspondence predict the transformer's
behavior from what we can infer for informed Bayes?

Earlier work points out the following. The prediction over a new test sequence (prompt) was called in-context learning and has been explained as
inference about a hidden source \citep{xie2022explanation}. Networks trained
on samples from a prior learn to predict like the posterior
\citep{muller2022transformers}. Regression experiments find Bayesian behavior
within and across task families \citep{panwar2024bayesian}.
Information-theoretic work separates what is learned across tasks from what
is inferred within a new task \citep{jeon2024information}. Meta-trained transformers and LSTMs were shown to reach the universal predictor of their training class \citep{graumoya2024universal}. For sequential
sources, the closest studies use Markov chains with independently drawn
transition rows \citep{edelman2024evolution,bondaschi2025markovlaplace},
smoothing across context orders \citep{dangelo2026induction}, variable-order
context trees \citep{zhou2024vomc}, and hidden causal structures
\citep{wei2026causal}. None of this guarantees that a trained network reaches the Bayesian optimum; we measure in bits how close it gets when contexts are tied together (Appendix~\ref{app:related} compares these papers in detail).

In the experiments reported in this paper each of the sources tests one kind of shared structure. Coins test whether
the prior on one unknown bias is learned. In product sources every token is
a pair of independent coordinates. The table that maps pairs to token
labels, the \emph{encoding}, is either known or hidden. In the shared-base
source all transition rows are built from one unknown base law and one
unknown depth, so tokens seen in one context help in every other context.
The uniform-base source, with a known uniform base law, is a control, and
in pooled training the family itself is unknown. The Bayesian predictors are
built from the row integrals of the layered simplex architecture (LSA)
\citep{feder2026lsa}.

Our main finding is that informed Bayes predicts, in bits, how a trained transformer responds in four tests that deviate from the training conditions (Section~\ref{sec:context-application}; Table~\ref{tab:pred}). In each test we first work out what informed Bayes does, from the family and prior alone, and then measure the transformer. Two tests change the test source: permuting the token labels of a product source, where Bayes predicts the loss that remains however long the sequence is, and replacing one or both transition matrices mid-sequence, where Bayes predicts that two new matrices cost twice as much as one. Two stay inside the training family: whether tokens seen in other contexts help when a token appears for the first time, and at which source depth copying an earlier continuation starts to pay. The transformer meets all four predictions.

The precondition for these predictions is that the transformers learn the shared structure at all, and Section~\ref{sec:bayes-evidence} shows that they do. Two-layer models are close to informed Bayes when the structure is visible in the tokens: for coins, for the known encoding and for the uniform base. When the encoding or the base law is hidden, two layers are not enough, and eight-layer models trained on four times as many sequences close most of the remaining gap.

Section~\ref{sec:why} asks why the transformer can imitate (approximately) the informed Bayes: the training objective is minimized by informed Bayes whose prediction structure is a function of some statistics of the prefix that attention can select and accumulate, leaving only a map from, e.g., counts to a forecast to be learned. In the framework where an architecture is a prior over predictors \citep{feder2025framework}, training pays the complexity of informed Bayes under that transformer architecture's prior, while informed Bayes pays the complexity of the source under the family's prior. What the sources examples discussed here add is that these statistics are global rather than local, which explains which sources need more capacity and training, and where the residual gap sits.

\section{Sources and prediction}
\label{sec:sources-prediction}

A \emph{source family} is a set of sources $\{P_\theta\}_{\theta\in\Theta}$.
Each $P_\theta$ assigns a probability to every token sequence. To generate one
sequence, we draw the parameters $\theta$ from a prior $\pi$ on $\Theta$ and
then draw the tokens from $P_\theta$. Every sequence gets a fresh $\theta$.

\subsection{The source families}
\label{sec:family-sources}

All sources except the coins are first-order Markov chains; the
\emph{context} of a prediction is the previous token, and its
\emph{transition row} is the distribution of the next token.
Table~\ref{tab:sources-at-a-glance} in Appendix~\ref{app:strategies}
summarizes, for each source, what all sequences share and what each sequence
draws afresh.

\paragraph{Coins.}
The tokens are zeros and ones. Draw a bias $\theta$ once per sequence,
then draw the tokens independently with probability $\theta$ of one.
We use two priors in separate training conditions:
\[
 \Pr\{\theta=0.2\}=\Pr\{\theta=0.8\}=\tfrac12,
 \qquad\text{or}\qquad \theta\sim\operatorname{Unif}[0,1].
\]

\paragraph{LSA row prior.}
The other sources draw their transition rows from the following prior
\citep{feder2026lsa}: draw $L$ independent vectors uniformly from the
simplex on $d$ symbols, multiply them entry by entry, and normalize. We
write $\mathcal D_{d,L}$ for the resulting distribution and call $L$ its
\emph{depth}. For $L=0$ the vector is uniform, for $L=1$ uniform on the
simplex, and the larger $L$, the more it concentrates on a few symbols.

\paragraph{Product source: known encoding.}
Each token is a pair $(u,v)$ with $u\in\{1,\ldots,d_1\}$ and
$v\in\{1,\ldots,d_2\}$. A fixed and known table, the encoding, maps the
$d=d_1d_2$ pairs to token labels. The two coordinates follow independent
Markov chains, so
$P_\theta(u',v'\mid u,v)=T_1(u'\mid u)T_2(v'\mid v)$.
For each coordinate $i$, draw a depth $L_i$ uniformly from
$\{1,\ldots,L_{\max}\}$, and then draw every row of $T_i$ independently from
$\mathcal D_{d_i,L_i}$; the initial coordinates are independent and uniform.
We use $d_1=d_2=8$ and $L_{\max}=8$. The factorization helps because two
$8\times8$ matrices are learned from far fewer tokens than one $64\times64$
matrix.

\paragraph{Product source: hidden encoding.}
\label{sec:hidden-factor-source}
Here the encoding is hidden, and it changes from sequence to sequence. For
each sequence, draw every row of $T_1$ and of $T_2$ from $\mathcal D_{d_i,1}$,
draw a one-to-one table
$\sigma:\{1,\ldots,d_1\}\times\{1,\ldots,d_2\}\to\{1,\ldots,d_1d_2\}$
uniformly at random, and run the two coordinate chains $(U_t)$ and $(V_t)$
from independent uniform starting states. The observed token is
$X_t=\sigma(U_t,V_t)$, so
$P_\theta\bigl(\sigma(u',v')\mid\sigma(u,v)\bigr)
=T_1(u'\mid u)T_2(v'\mid v)$.
The matrices and the encoding stay fixed within one sequence and are drawn
anew for the next, so the same label can stand for different pairs in
different sequences. Informed Bayes knows $d_1$, $d_2$ and the sampling
rule, but none of the drawn quantities; it integrates over the matrices in
closed form and sums over all encodings
(Appendix~\ref{app:hidden-factor-bayes}). We use $d_1=2$ and $d_2=8$, giving
$16$ labels.

\paragraph{Uniform-base source.}
\label{sec:uniform-base-source}
Set $d=64$, draw one depth $L$ uniformly from $\{0,\ldots,8\}$, draw each
of the $64$ rows independently from $\mathcal D_{d,L}$, and start from a
uniform token. The rows share the depth and nothing else. This is the case
$\alpha=0$ of the next source, with a known uniform base law.

\paragraph{Shared-base source.}
\label{sec:shared-base-source}
Tokens are $\{1,\ldots,d\}$ with $d=64$. Draw independently an exponent
$\alpha$ uniformly from $[0.5,1.5]$, a depth $L$ uniformly from
$\{0,\ldots,8\}$, and a uniform random permutation $\sigma$ of the tokens.
The \emph{base law} $q$ is a Zipf distribution with exponent $\alpha$, and
$\sigma$ assigns its ranks to the tokens:
\[
 q_{\sigma(r)}=\frac{r^{-\alpha}}{\sum_{j=1}^d j^{-\alpha}},
 \qquad r=1,\ldots,d.
\]
For each context $a$, draw a probability vector $V_a$ from
$\mathcal D_{d,L}$ independently, and set the transition row:
\[
 T(b\mid a)=\frac{q_bV_a(b)}{\sum_cq_cV_a(c)}.
\]
Draw $X_1$ from $q$ and $X_{t+1}$ from $T(\cdot\mid X_t)$. One base law, one
depth and one transition matrix are used for the whole sequence. For $L=0$
every row equals $q$, and the tokens are independent. For $L>0$ every row is
the shared base law $q$, reshaped by the row's own vector $V_a$. Because $q$
and $L$ are shared by all rows, tokens seen in one context help prediction
in every other context.

\paragraph{Pooled families.}
For each sequence, draw one of the three $64$-symbol families (product with
known encoding, uniform base, shared base) with probability $1/3$ each, and
then draw a source from that family as above. Neither the family nor the
parameters are revealed.

\subsection{What the predictors know}
\label{sec:family-predictors}

The \emph{transformer} is trained by next-token prediction on sequences
drawn from its family and prior, never seeing the parameters. With its
weights fixed, on a test sequence it sees $x^t=(x_1,\ldots,x_t)$ and
predicts token $t+1$; the first token is given and not scored.

\emph{Informed Bayes}, written $B$, is told the training family and prior
but not the parameters of the test sequence. Its probability for a sequence
and for the next token are
\begin{equation}
 Q_B(x^n)=\int P_\theta(x^n)\,\pi(d\theta),
 \qquad
 Q_B(b\mid x^t)=\frac{Q_B(x^tb)}{Q_B(x^t)}.
 \label{eq:family-bayes}
\end{equation}
Equation~\eqref{eq:family-bayes} averages over all source parameters using the likelihood of the entire sequence. Parameters shared across context rows appear in several factors of this likelihood, so observations in one context affect predictions in another, with no separate parameter estimation. In pooled training, $\theta$ also includes the unknown family, so the same formula averages the three families' predictions with their posterior weights. Appendix~\ref{app:strategies} gives the formulas and numerical methods for each source.

\subsection{Prediction accuracy}
\label{sec:prediction-accuracy}

All losses are in bits. Let $Q_M(\cdot\mid x^t)$ be the next-token
probabilities of a predictor $M$. We write $M=\mathrm{TF}$ for a trained
transformer and $M=B$ for informed Bayes; thus $Q_{\mathrm{TF}}$ is the
transformer's predictive distribution. Let $\pi_{\mathrm{test}}$ denote
the distribution of source parameters used in the test. It equals the
training prior $\pi$ in full-prior tests; in fixed-setting tests, the
specified parameters are held fixed. The \emph{regret} of $M$ after $t$
observed tokens is the expected extra loss compared with a predictor that
knows the parameters $\theta$ of the sequence:
\begin{equation}
 r_M(t)=\mathbb E_{\substack{\theta\sim\pi_{\mathrm{test}}\\X^t\sim P_\theta}}
 \left[D_{\mathrm{KL},2}\!\left(
 P_\theta(\cdot\mid X^t)\Vert Q_M(\cdot\mid X^t)
 \right)\right].
 \label{eq:instantaneous-regret}
\end{equation}
Here $D_{\mathrm{KL},2}$ is the Kullback--Leibler divergence in bits. The
expectation first draws the source parameters from $\pi_{\mathrm{test}}$
and then the observed sequence from that source. The \emph{gap to Bayes}
of the transformer is
\begin{equation}
 \Delta(t)=r_{\mathrm{TF}}(t)-r_B(t)
 =\mathbb E_{\substack{\theta\sim\pi_{\mathrm{test}}\\X^{t+1}\sim P_\theta}}\!\left[
 \log_2\frac{Q_B(X_{t+1}\mid X^t)}
 {Q_{\mathrm{TF}}(X_{t+1}\mid X^t)}
 \right].
 \label{eq:transformer-bayes-gap}
\end{equation}

In a \emph{full-prior} test the parameters of each test sequence are drawn
from the training prior $\pi$; then the next token given $X^t$ has
distribution $Q_B(\cdot\mid X^t)$ by \eqref{eq:family-bayes}, and
\begin{equation}
 \Delta(t)=\mathbb E_{\substack{\theta\sim\pi\\X^t\sim P_\theta}}
 \left[D_{\mathrm{KL},2}\!\left(
 Q_B(\cdot\mid X^t)\Vert Q_{\mathrm{TF}}(\cdot\mid X^t)
 \right)\right]\geq0.
 \label{eq:bayes-optimality}
\end{equation}
In a full-prior test no predictor has a smaller expected loss than informed
Bayes, and the gap measures how far the two predictive distributions are
apart. The identity holds with any predictor in place of
$Q_{\mathrm{TF}}$; this is the sense in which $Q_B$ is the \emph{universal
probability} for the source family \citep{merhav1998universal}. In a
\emph{fixed-setting} test we fix some parameters, for example $\alpha=1$
and $L=4$, and draw the others from the prior; informed Bayes still uses the
full prior, and the gap can then be negative.

Both predictors are evaluated on the same test sequences; differences are formed within each sequence before averaging, and each $\pm$ denotes one standard error of the mean across sequences. Cumulative quantities sum over all scored predictions (Appendix~\ref{app:prediction-metrics}).

\section{How close do transformers get to informed Bayes?}
\label{sec:bayes-evidence}
\label{sec:bayes-comparison}

Figure~\ref{fig:bayes-summary} compares trained transformers with informed
Bayes on six sources, from a prior on one coin bias to a mixture of
families. Where shared structure is visible, two-layer transformers are
close to informed Bayes, and increasing capacity brings them closer still.
When the structure is hidden, more layers and training close most of the
gap. Table~\ref{tab:gap-summary} in Appendix~\ref{app:gap-table} summarizes the gaps; Appendix~\ref{app:bayes-evidence-details} gives the models, training, test sets and numerical checks.

\begin{figure}[t]
\centering
\includegraphics[width=0.8\linewidth]{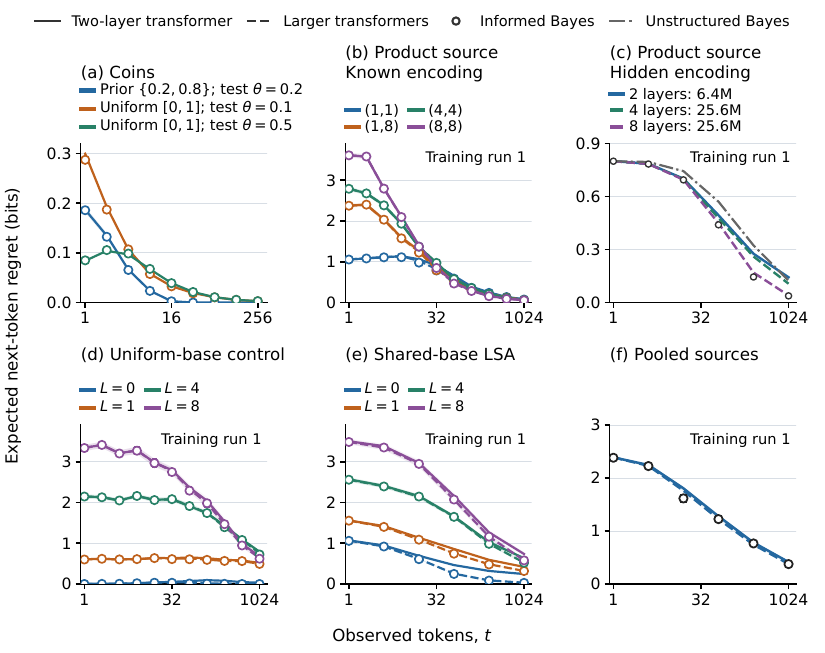}
\caption{\textbf{Transformers and informed Bayes on six sources.}
Expected next-token regret on held-out test sequences. Solid curves show
two-layer transformers; dashed curves show four layers in (b,d,f), four or
eight in (c), and eight in (e). Open circles show informed Bayes. Colors
distinguish source settings in (a,b,d,e) and models in (c).
(a) Coins: two-point prior at $\theta=0.2$; uniform prior at
$\theta=0.1,0.5$.
(b) Known $8\times8$ product encoding at four depth pairs.
(c) Hidden $2\times8$ encoding. Gray dash-dot ``Unstructured Bayes'' treats
the $16$ symbols as one Markov chain with independent transition rows, each under a uniform Dirichlet prior, ignoring
the product factorization.
(d) Uniform base and (e) shared base with $\alpha=1$, at $L=0,1,4,8$.
Bayes points in (e) extend through $t=1024$ at all four depths.
(f) Equal mixture of three $64$-symbol families; Bayes points through $t=1024$.
Panels (b)--(f) show run~1. Appendix~\ref{app:bayes-evidence-details}
gives both runs and protocols, including the additional training for
larger models in (c,e); budgets are equal within (b,d,f).}
\label{fig:bayes-summary}
\end{figure}

\subsection{A source prior: coins}
\label{sec:bayes-coins}

We train one two-layer model on biases $0.2$ and $0.8$ with equal
probability, and another on biases drawn uniformly from $[0,1]$. Both follow
informed Bayes at all tested biases and all $t$ (panel (a);
Appendix~\ref{app:coins-comparison} shows the other biases).

What is the prior worth? We test both models on the same sequences with bias
$0.2$ or $0.8$. Over $256$ predictions, the two-point-prior model loses $2.293\pm0.015$
bits less than the uniform-prior model; for the two Bayesian predictors the
difference is $2.290\pm0.015$ bits. The transformer has learned what the
prior is worth.

\subsection{A known encoding: product source}
\label{sec:bayes-product}

With a known $8\times8$ encoding, the two-layer transformer is close to
informed Bayes (panel (b)). Four-layer models trained on the same number of
sequences reduce the largest gap from $0.124$ bits to $0.044$ and $0.049$
bits in two runs (Appendix~\ref{app:known-product-comparison}). The Bayesian curve also
predicts how many tokens are needed to reach a target regret
(Appendix~\ref{app:application-context}).

\subsection{Discovering a factorization: hidden encoding}
\label{sec:bayes-hidden-factor}

When the encoding is hidden, the transformer must find out in each sequence
which labels belong together (panel (c)). After $1024$ tokens, the two-layer
models have regrets of $0.1430$ and $0.1335$ bits, near the $0.124$ bits of
unstructured Bayes, which ignores the product factorization. Eight-layer
models with more training reach $0.0402$ and $0.0397$ bits, close to the
$0.0378$ bits of informed Bayes. Four-layer models lie in between;
Appendix~\ref{app:hidden-factor-experiment} gives the full capacity and training comparisons.

\subsection{A known base law: uniform-base source}
\label{sec:bayes-uniform-base}

With the base law known, increasing capacity brings the transformer close
to informed Bayes at all four depths (panel (d)). With the same training
budget, four layers reduce the largest gap from $0.087$ bits to $0.037$
and $0.041$ bits in two runs. At $L=0$ the tokens are independent and
informed Bayes soon predicts uniformly; the four-layer models learn this
behavior too (Appendix~\ref{app:uniform-base-comparison}).

\subsection{An unknown shared base law: shared-base source}
\label{sec:bayes-shared-base}

With the base law unknown, eight-layer models with more training close most
of the two-layer gap (panel (e)). For both
runs, all four depths and all tested $t\le64$, the gap is below $0.007$
bits. At $t=256$ and $t=1024$, its absolute value is below $0.0052$ bits
at all four depths. All shared-base tests use $\alpha=1$;
Appendix~\ref{app:shared-base-comparison} gives the capacity and training
comparisons.

\subsection{Inferring the source family: pooled training}
\label{sec:bayes-pooled}

Here neither predictor is told which family generated the sequence (panel
(f)). We train on an equal mixture of the three $64$-symbol families;
informed Bayes combines their predictions with posterior family weights.
Panel (f) shows Bayes points through $t=1024$; the paired comparisons reported here use $t=1,4$ (Appendix~\ref{app:pooled-comparison}). On $1024$ new sequences per family, they show that the transformer reproduces this mixture prediction. With the same
training sequences, four-layer models reduce the gap after four tokens from
$0.0197$ and $0.0190$ bits to $0.0052$ and $0.0025$ bits in the two runs
(Appendix~\ref{app:pooled-comparison}).

Appendix~\ref{app:application-family} compares the cumulative cost of not
knowing the family with its Bayesian prediction.
\FloatBarrier

\section{Four predictions beyond the learning curve}
\label{sec:context-application}

\begin{table}[t]
\caption{\textbf{Predicted versus measured.} Each row is a quantity that informed Bayes fixes from the family and prior alone, next to the trained transformer's value on the same test sequences. Bits; $\pm$ one standard error; ``Bayes'' is the reference's own value on the same sequences where it is a random quantity.}
\label{tab:pred}
\centering\footnotesize
\setlength{\tabcolsep}{3pt}
\begin{tabular}{@{}lllll@{}}
\toprule
Test (\S) & Quantity & Predicted & Bayes & Transformer \\
\midrule
Prior (\ref{sec:bayes-coins}) & bits saved by the two-point prior (256 tokens) & --- & $2.290{\pm}0.015$ & $2.293{\pm}0.015$ \\
Mismatch (\ref{sec:application-mismatch}) & residual regret, depths $(1,1)$, $t{=}1024$ & $0.966{\pm}0.003$ & 1.039 & 1.040 \\
 & residual regret, depths $(4,4)$, $t{=}1024$ & $2.237{\pm}0.007$ & 2.299 & 2.309 \\
Change (\ref{sec:application-changes}) & extra regret, two new matrices, $t{=}1024$ & $0.345$ ($2\times0.173$) & --- & 0.350 \\
Unseen (\ref{sec:application-transfer}) & gain over uniform, shared base, $L{=}1$, $t{=}64$ & --- & 0.799 & 0.793 \\
 & gain over uniform, uniform base, any $L,t$ & 0 & 0 & $<0.0005$ \\
Copying (\ref{sec:application-copying}) & gain from token order, $L=0$ & $\le0$ & $<0$, all $t$ & $<0$, all $t$ \\
 & gain from token order, $L{=}3$, $t{=}64$ & --- & $0.262{\pm}0.027$ & $0.263{\pm}0.028$ \\
\bottomrule
\end{tabular}
\end{table}

We now test trained transformers in four further, more challenging, ways. In every test the transformer keeps its weights, informed Bayes keeps the training family and prior, and neither is told anything about the test. Each test yields a number that the family and prior fix in advance; Table~\ref{tab:pred} collects them next to the measured values, Figure~\ref{fig:applications} shows the curves, and Appendix~\ref{app:application-protocols} gives derivations and all results.

\begin{figure}[t]
\centering
\includegraphics[width=0.8\linewidth]{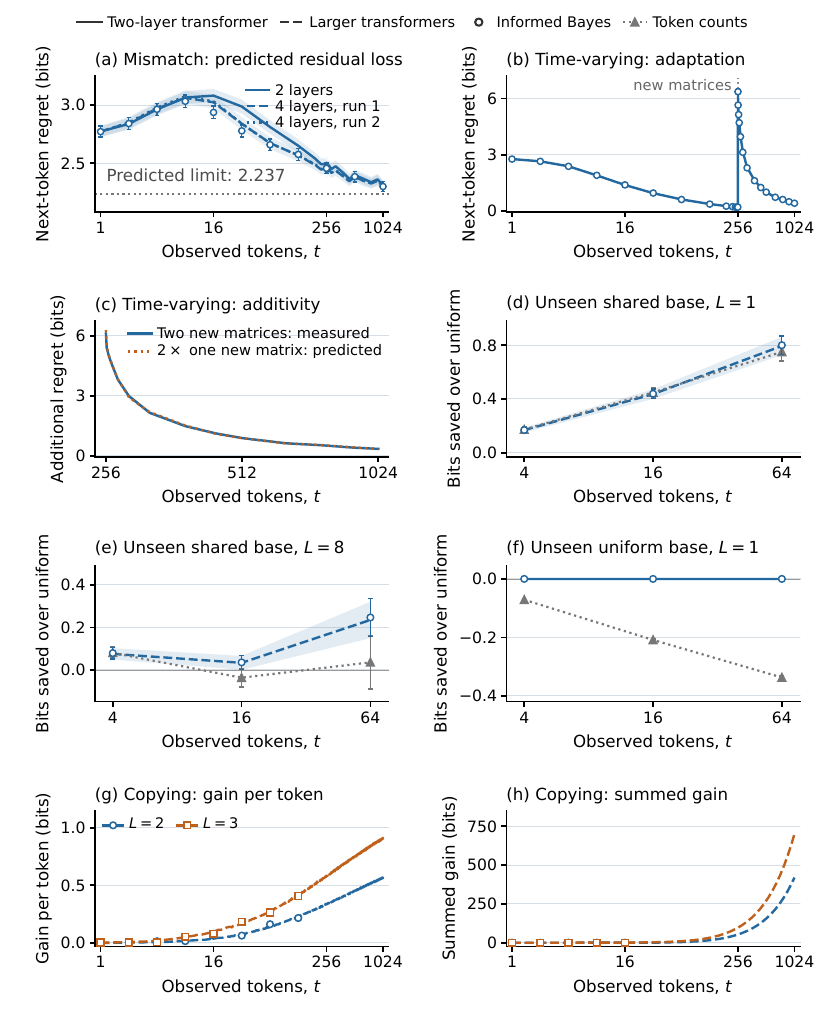}
\caption{\textbf{Four further tests of trained transformers against informed
Bayes.}
(a) Regret under randomly permuted product labels; gray dotted line:
predicted limit $R_\infty$.
(b,c) New transition matrices at $t=256$: regret with two new matrices (b);
additional regret with two versus twice that with one (c).
Panels (a)--(c) use depths $(4,4)$.
(d)--(f) Bits saved over uniform prediction when the current token is new:
shared base at $L=1,8$ (d,e); uniform base at $L=1$ (f).
(g,h) Run-1 shuffled-trained minus ordered-trained loss, per token (g) and
cumulative (h), compared with the Bayesian gain over its frequency-only
predictor through $t=16$, and at $t=32,64,128$ in (g).
At $t=128$, Bayes uses $256$ sequences for each of $L=2,3$.
Curves show transformers; open markers show Bayes in matching colors.
Solid curves show two layers; dashed curves show four layers in (a,g,h)
and eight in (d,e). The second four-layer run in (a) is dotted.
Gray curves show controls, including token counts in (d)--(f).
In (g,h), blue circles denote $L=2$ and orange squares $L=3$.
Horizontal axes use base-2 scales; shared-base panels use $\alpha=1$.
Intervals are pointwise 95\% in (a)--(c), (g) and (h); markers in (d)--(h) show
one sequence standard error.}
\label{fig:applications}
\end{figure}

\subsection{Source mismatch}
\label{sec:application-mismatch}

The transformer was trained on product sources with a known $8\times8$
encoding. In the test we apply a random permutation to the $64$ labels of a
product source, fixed throughout the sequence, which neither predictor is
told. Both still read every label as a pair $(U,V)$ with the training
encoding, but under the permutation the two coordinates are no longer
independent chains. Informed Bayes can only fit a product of two chains, and
its regret approaches
\begin{equation}
 R_\infty=H(U'\mid U)+H(V'\mid V)-H(U',V'\mid U,V),
 \label{eq:mismatch-product-limit}
\end{equation}
where $(U',V')$ follows $(U,V)$ and the entropies are those of the test
source: the information a product of two chains cannot use.

After $1024$ tokens, the two-layer transformer closely matches informed
Bayes at both tested depths (panel (a)). Increasing capacity brings the
intermediate predictions closer too: with the same training budget, both
four-layer runs reduce the depth-four gap at $32$ tokens from $0.209$ to
$0.062$ bits, about $70\%$. Appendix~\ref{app:product-shifts} gives the
predicted limits and both runs at both depths.

\subsection{Time-varying source}
\label{sec:application-changes}

The transformer was trained on product sources that do not change, and
informed Bayes uses that family and prior. The test source starts as a
product source from the family. After $256$ tokens, one coordinate or both
coordinates switch to a new transition matrix, drawn independently from the
same prior. The current state and the encoding carry over. The predictors see only the tokens; they are not reset and are not told about the change.

Informed Bayes keeps updating as if the source had not changed. At each
prediction, both the true next-token law and Bayes's forecast factor across
the two coordinates. Their KL divergence therefore adds across coordinates.
Let $r_u(t)$ be the expected regret of one coordinate when its matrix stays
unchanged, and $r_c(t)$ when it is redrawn. The coordinates have identical
laws and priors, and changing one affects neither the data nor the forecast
for the other. Consequently, the expected regrets with no, one, and two new
matrices are $2r_u(t)$, $r_u(t)+r_c(t)$, and $2r_c(t)$, respectively.
Subtracting the no-change regret gives, at every $t$ after the change:
\begin{equation}
 r_B^{\mathrm{both}}(t)-r_B^{\mathrm{stationary}}(t)
 =2\bigl[r_B^{\mathrm{one}}(t)-r_B^{\mathrm{stationary}}(t)\bigr].
 \label{eq:product-shift-additivity}
\end{equation}

Panel (b) shows close transformer--Bayes agreement before the switch, through the jump at the change, and during the subsequent decline.
The additional regret also agrees with the factor of two (panel (c)).
Simultaneous $95\%$ bands for the paired differences include zero at both
depths and all $14$ tested times after the change. Any predictor that treats the coordinates separately and
symmetrically has this doubling. Panel (c) therefore tests whether the
transformer retains the factorization; panel (b) tests whether its regret
matches informed Bayes. Appendix~\ref{app:product-shift-prediction} gives
the derivation and full comparison.

\subsection{An unseen context}
\label{sec:application-transfer}

We look at the predictions made when the current token appears for the
first time, so that nothing has been observed of its transition row, and
measure the bits saved over a uniform prediction; a predictor that uses only
token counts is added for comparison. In the shared-base source informed
Bayes can still predict, because transitions out of other tokens tell it
about the shared base law; in the uniform-base source the rows share no base
law and it predicts uniformly.

The transformers do the same (panels (d--f)): both eight-layer shared-base models are within $0.015$ bits of informed Bayes at every tested depth and context length. The token-count predictor is
close to Bayes at small depth, where token frequencies approximate the base
law, but at $L=8$ and $t=64$ it loses $0.211\pm0.049$ bits to Bayes while
the transformer loses $0.012\pm0.011$. In the uniform-base source the
transformer predicts almost uniformly, while token counts increase loss.
Each transformer thus uses the other contexts exactly when its family makes
them informative (Appendix~\ref{app:application-transfer}).

\subsection{When copying starts to pay}
\label{sec:application-copying}

When token $a$ appears again, a copying predictor raises the probability of
the token that followed $a$ before. This uses the order of the tokens, which
a frequency-only predictor cannot. At depth $L=0$ every row of the
shared-base source equals the base law, so the order is worth nothing and
looking for it costs; at positive depth the order helps once enough tokens
have been seen. We ask whether copying emerges in the transformer where
informed Bayes says it should. Next to informed Bayes $B$ we use a Bayesian
predictor $B_{\mathrm{freq}}$ that keeps the same prior and first-token law
but models every later token as an independent draw from the stationary law
of the source (Appendix~\ref{app:shared-shuffle-control}); it approximates
the training law of the shuffled-trained transformer below. With $\ell_M(t)$
the loss of predictor $M$ on token $t+1$, the gain from the order and its
running sum are
\begin{equation}
 g_B(t)=\mathbb E[\ell_{B_{\mathrm{freq}}}(t)-\ell_B(t)],\qquad
 G_B(n)=\sum_{t=1}^{n}g_B(t).
 \label{eq:shared-shuffle-payoff}
\end{equation}
For the transformer we compare a matched pair of four-layer models. One
sees ordered training sequences; the other sees positions
$2,\ldots,1025$ shuffled within each sequence, preserving the first token
and all token counts. It can learn frequencies but cannot learn to copy.
Both are tested on the same ordered sequences; their loss difference $g_{\mathrm{TF}}(t)$ and its running sum $G_{\mathrm{TF}}(n)$ measure what copying is worth to the transformer.

Copying emerges where Bayes says it should. At $L=0$ both gains are negative at every measured context. At $L=2$ and $L=3$ both gains become positive, larger
at the greater depth, and of the same size at the measured contexts. At
$L=3$ and $t=64$, Bayes gains $0.262\pm0.027$ bits and the run-1
transformer gains $0.263\pm0.028$ bits.
The shuffled-trained transformer matches
$B_{\mathrm{freq}}$ within $0.0043$ bits at all sixteen tokens, so it has
learned the frequencies and nothing more. At longer contexts the
transformer's gain keeps growing, to $0.57$ and $0.91$ bits per token at
$t=1024$ for $L=2,3$ (panels (g,h)); Appendix~\ref{app:shared-shuffle-control} gives both runs and matched training settings.
\FloatBarrier

\section{Why a transformer can approximately imitate informed Bayes}
\label{sec:why}

Three facts make the imitation possible. First, informed Bayes minimizes
expected training loss: for any predictor $Q$, when source parameters follow
the training prior $\pi$,
\begin{equation}
\begin{aligned}
\mathbb E_{\substack{\theta\sim\pi\\X^{t+1}\sim P_\theta}}
 \left[-\log_2 Q(X_{t+1}\mid X^t)\right]
&=\mathbb E_{\substack{\theta\sim\pi\\X^{t+1}\sim P_\theta}}
 \left[-\log_2 Q_B(X_{t+1}\mid X^t)\right]\\
&\quad+\mathbb E_{\substack{\theta\sim\pi\\X^t\sim P_\theta}}
 \left[D_{\mathrm{KL},2}\!\left(
 Q_B(\cdot\mid X^t)\Vert Q(\cdot\mid X^t)
 \right)\right],
\end{aligned}
\label{eq:objective}
\end{equation}
because $X_{t+1}\mid X^t$ has distribution $Q_B(\cdot\mid X^t)$.
The divergence vanishes only when $Q=Q_B$, so $\Delta(t)$ in
\eqref{eq:bayes-optimality} measures the remaining excess training loss without
requiring parameter labels. Ignoring structure means using the universal
probability of a larger family and paying its larger regret
\citep{merhav1998universal}, as in Figure~\ref{fig:bayes-summary}(c). The same argument explains why transformers and LSTMs meta-trained on random tree sources reach, e.g., the loss of the universal context-tree-weighting mixture \citep{graumoya2024universal}.

Second, informed Bayes needs memory even for memoryless sources: coin
prediction uses the counts of zeros and ones. Shared-base prediction uses
all rows; pooled prediction also uses each family's evidence. A transformer
reads the prefix and can approximate this growing state. Finite-state
prediction bounds relate the remaining regret to the number of states
\citep{meron2004finite,ingber2006prediction,dar2011finite}; likewise, a
transformer's context, width, depth and training limit its approximation.
In the framework of \citet{feder2025framework}, the randomly initialized architecture is a prior over predictors and the sequences needed to reach a target are set by the target's complexity under that architecture prior; here the target is $Q_B$, so the transformer pays two universality costs: across sequences for $Q_B$ relative to the architecture, and within a sequence for $\theta$ relative to the family; the gaps of Sections~\ref{sec:bayes-evidence},\ref{sec:context-application} measure the first.

Third, Bayesian prediction decomposes as
$Q_B(b\mid x^t)=F(S(x^t))_b$: sufficient statistics $S$ followed by a
forecast map $F$. For the independent-row chain of
\citet{edelman2024evolution}, $S$ counts successors in the current context
and $F$ is the add-one rule. Attention to earlier occurrences of the current
token supplies \emph{selection} and \emph{accumulation}; the feed-forward
network learns the prior-specific \emph{estimator} $F$. Here $S$ is global:
the shared base and depth use all rows, the hidden partition must be inferred
from the count matrix, and family selection uses a softmax of accumulated
evidence. Appendix~\ref{app:why} develops this decomposition, its relation
to architecture-based complexity, and its implications for capacity and
training.

\section{Discussion}
\label{sec:conclusion}

Training on freshly drawn source parameters teaches transformers what sequences share, and their predictions approach informed Bayes, with larger models and more training needed when the encoding or base law is hidden. From the family and prior alone, informed Bayes predicts the costs of a wrong factorization and a source change, the benefit of other contexts for a new token, and when copying starts to pay. A trained network can fail each prediction, so these tests measure what it has learned and give circuit and training-dynamics studies a target \citep{raventos2023pretraining,falck2024martingale}.

The same account makes predictions for language models, where the family is not specified: documents sharing latent structure across contexts are the shared-base case, and the unseen-context test says when a model should generalize to a context never seen continued; retrieval of a first-occurrence continuation is the copying test, with a depth-dependent onset; the mismatch test says that a model keeps paying its training family's price on out-of-distribution text however long the prompt, and the change test that adaptation to a shift is as slow as a stationary posterior. When pretraining data are not fresh draws from a stable class \citep{raventos2023pretraining,chan2022data,falck2024martingale} there is no informed Bayes to imitate, and that is where the correspondence should break.

Our evidence is limited to synthetic sources, alphabets of at most $64$ symbols, at most $26$ million parameters, some baselines with one training run, and shared-base tests at $\alpha=1$; the hidden-encoding and shared-base comparisons change size, training amount and schedule together. Agreement in loss does not establish posterior computation. Section~\ref{sec:why} suggests a design to test: learned selection, exact accumulation, and the family's closed-form estimator.

\label{end:main}

\section*{AI use statement}
We used large language models to write and test the code for the sources,
the Bayesian predictors, the training and the evaluation, to draft and edit
the text, to search the literature and to make the figures. They did not
propose experiments. The authors checked all derivations, numerical results
and citations, and take full responsibility for the content of the paper.

\section*{Reproducibility statement}
Appendix~\ref{app:strategies} defines all Bayesian predictors. Appendices~\ref{app:bayes-evidence-details} and~\ref{app:application-protocols} give the models, the training settings, the test sets and the numerical checks, including the numerical error allowances applied to every reported Bayesian value. The code is available at \url{\coderepo}.

\clearpage
\bibliographystyle{iclr2027_conference}
\bibliography{refs}
\clearpage
\appendix
\raggedbottom

\section{Related work: Bayesian prediction and learned source structure}
\label{app:related}

This paper brings together four lines of work: Bayesian accounts of
in-context learning, transformer prediction on sequential sources,
hierarchical models that share information across contexts, and universal
prediction under log loss. The Bayesian interpretation itself is established
in prior work. Our use of it is quantitative: for fully specified structured
source priors, we compute the informed posterior predictive distribution at
every number of observed tokens and ask whether it predicts both transformer loss and
the response to controlled interventions.

\subsection{Bayesian accounts of in-context learning}

\citet{xie2022explanation} model pretraining data as a mixture of latent
concepts and explain in-context learning as implicit Bayesian inference of the
concept generating the prompt. Their analysis gives asymptotic guarantees
for a hidden-Markov construction. \citet{muller2022transformers} take a
complementary amortized-inference approach: prior-data fitted networks are
trained on samples from a specified prior and learn its posterior predictive
map. Together, these works establish the central possibility used here:
training examples can convey a prior, while the prompt updates beliefs about
the latent parameters of a new example.

Several subsequent results sharpen that connection. In regression,
\citet{panwar2024bayesian} find that sufficiently capable transformers track
tractable Bayesian predictors across individual function families and
mixtures of families, while also documenting departures from the Bayesian
rule outside the training families. \citet{zhang2025bayesian} show that perfect pretraining
implements Bayesian model averaging under their dynamic prompt model, and
separate the resulting in-context error from approximation and pretraining
errors. \citet{jeon2024information} decompose Bayes-optimal prediction error
into learning the task distribution across sequences and learning the current
task within a sequence. This distinction corresponds directly to our setup:
the transformer must acquire the source family and prior during training, then
infer the freshly drawn parameters from the test prefix.
\citet{bai2023transformers} provide a different, constructive route: a
transformer can run several statistical procedures in context and select
among them using validation. Our pooled predictor averages source
families with their posterior probabilities, so both constructions address
algorithm selection but implement different decision rules.

Bayesian behavior cannot be inferred from good average loss alone.
\citet{raventos2023pretraining} show that the
effective in-context estimator in linear regression changes with pretraining
task diversity. \citet{falck2024martingale} derive a necessary martingale
condition for Bayesian prediction in exchangeable settings and find
violations in language models. We therefore treat informed Bayes as an
optimal, numerically testable reference rather than as an assumption about a
trained network. Relative to these works, our focus is next-token prediction
from dependent discrete sequences, with the complete training prior specified
and with transformer and Bayesian losses evaluated on the same prefixes. The
source interventions in Section~\ref{sec:context-application} test additional
consequences of the same posterior predictive rule rather than relying on a
single agreement in average loss.

\subsection{Markov sources, induction heads, and learned estimators}

The copying mechanism associated with Markov prediction has a well-developed
mechanistic literature. \citet{elhage2021framework} identify induction-head
circuits that use an earlier occurrence of a token to retrieve its
continuation. \citet{olsson2022context} connect the formation of such heads
to a sharp change in in-context learning during training, with causal evidence
in small attention-only transformers and correlational evidence in larger
models. \citet{singh2024induction} causally isolate interacting subcircuits
that drive induction-head formation, while
\citet{induction2025predicting} relate formation time to repetition,
continuation reliability and the amount of training data. These works
motivate our experiments on token order, but our claims concern predictive
distributions and losses, not the presence of a particular attention
circuit.

Random Markov sources make the learned estimator explicit.
\citet{edelman2024evolution} show a training progression from uniform to
unigram and then bigram prediction, and identify statistical induction heads
that estimate bigram probabilities in context. \citet{rajaraman2024markov}
show that a constant-depth transformer can represent the in-context
conditional empirical distribution for Markov sources of arbitrary fixed
order, while their attention-only construction composes induction heads.
\citet{bondaschi2025markovlaplace} show that a single-layer Mamba can implement
Laplace smoothing for Markov chains and compare this behavior with
transformers. \citet{akyurek2024context} likewise find that transformers
trained on random formal languages use in-context $n$-gram statistics. These
papers establish counting and smoothing of context-conditioned continuations
as learned sequential algorithms.

Other work lets the dependency structure itself vary. For variable-order
Markov chains, \citet{zhou2024vomc} use context-tree weighting as the
reference, construct a transformer that implements it exactly, and find that
two or more layers can learn the task empirically. \citet{wei2026causal}
study prompts with randomly sampled causal dependencies and show that a
two-layer transformer can implement, and trained models can approximate,
Bayesian model averaging over those structures. Their numerical transition
kernel is shared across the data set, whereas our source parameters are drawn
afresh for every sequence. \citet{dangelo2026induction}
show that induction heads need not merely count exact matches: finite
attention scales interpolate across partial context matches, while a
beginning-of-sequence token supplies Dirichlet-like pseudo-count smoothing.
Their hierarchical Dirichlet experiment also shows that trained models use
suffix-related evidence from other contexts. It compares learned smoothing
with count-based baselines; our comparison evaluates the full posterior
predictive distribution under a different shared hierarchy.
\citet{dangelo2025selective} show that transformers can select the relevant
lag when examples with different Markov orders are interleaved, locating an
explicit circuit for context-dependent structural selection.
Work on training dynamics explains how such rules arise. For example,
\citet{bietti2023birth} analyze the staged formation of associative memories,
and \citet{nichani2024causal} prove that gradient descent can learn latent
causal structure in a simplified prediction problem.

Our sequential experiments start from these results and change what must be
inferred. Product sources test a known or sequence-specific hidden
factorization; shared-base sources couple all transition rows through an
unknown base law and depth; pooled training adds uncertainty over the source
family. Rather than compare only with empirical counts or a generic smoothing
rule, we marginalize the stated prior and compare the transformer with that
finite-sample posterior predictive distribution. The result complements
representational and optimization analyses: it identifies the prediction
that full use of the training distribution would produce and separates that
output-level question from the circuit and training path that realize it.

\subsection{Hierarchical priors, shared structure, and transfer}

Sharing statistical strength across contexts is classical in language
modeling. \citet{mackay1995hierarchical} couple multinomial contexts through
a hierarchical Dirichlet model, and \citet{teh2006hierarchical} use
hierarchical Pitman--Yor processes to share information across $n$-gram
contexts while retaining power-law behavior. These models supply direct
precedents for our shared-base source: observations from one transition row
can improve prediction in another because the rows have common latent
parameters. Our row prior and latent variables differ, but the pooling
principle is the same and belongs to this hierarchical Bayesian tradition.

Transformers have also been studied as learners of an unknown population
prior. \citet{teh2025empiricalbayes} train transformers for Poisson empirical
Bayes and prove that sufficiently wide models can approach an oracle that
knows the prior, while small models perform competitively in experiments.
At the level of pretraining distributions, \citet{chan2022data} show that
properties such as within-sequence burstiness and token meaning across
sequences determine whether in-context learning emerges, and
\citet{raventos2023pretraining} show that task diversity changes whether a
network behaves like inference over a finite task set or like a continuous
prior predictor. \citet{deora2025occam} explain selection among Markov orders
through posterior weights that trade data fit against model complexity.
\citet{elmoznino2025occam} interpret next-token loss as prequential coding and
connect in-context hypothesis choice to an Occam tradeoff. These results all
emphasize that the structure of the training
distribution, rather than sequence data alone, determines the learned rule.

There are also direct precedents for two applications in
Section~\ref{sec:context-application}. \citet{genewein2023nonstationary}
train memory-based sequence models, including transformers, on piecewise
stationary sources and compare them with Bayesian algorithms that infer
latent change points and segment parameters. Our source-change experiment
freezes predictors trained on stationary sources, gives them no switching
model, and tests the additive response implied by two independent components.
Comparing early and cumulative log loss also belongs to the catch-up and
switching literature \citep{vanerven2007catching,vanerven2012catchup}. Our
ordered-versus-shuffled experiment uses this accounting to measure when
copying an earlier continuation pays for fixed trained predictors; it does
not propose a switching distribution.

The present paper makes that dependence directly measurable for sequential
sources. Every sequence redraws its local parameters, while the training
distribution fixes which coordinates, transition rows, and source families
are related. The informed Bayesian predictor therefore isolates the value of
each relation. The unseen-context experiment measures transfer through a
shared base law; the source-change experiment tests additivity across
independent latent components; the relabeling experiment measures the cost of
retaining a factorization outside its assumed family; and the pooled
experiment performs posterior averaging over source families. These are
controlled consequences of the stated source priors, rather than a claim that
hierarchical pooling or Bayesian model selection is new.

\subsection{Universal prediction and computable source priors}

Under log loss, the cumulative loss of a sequential probability assignment
is a codelength. A Bayesian mixture is optimal in expectation under its
stated prior. Universal and minimax codes instead control redundancy or
regret over a source class \citep{merhav1998universal,rissanen1984universal,
clarke1990information,xie2000asymptotic}. The
Krichevsky--Trofimov probability gives a canonical universal code for unknown
symbol probabilities \citep{krichevsky1981performance}, while context-tree
weighting averages efficiently over variable-memory sources
\citep{willems1995ctw}. Mixtures across source classes and hierarchical
universal coding extend this principle to uncertainty over the model itself
\citep{ryabko1984twice,feder1996hierarchical}; classical work
also connects compression to selection of Markov order
\citep{merhav1989order,csiszar2000bic}. Bayesian redundancy and cumulative
prediction risk are also characterized through mutual information
\citep{haussler1997mutual}. The minimum-description-length principle supplies
a complementary model-selection interpretation \citep{rissanen1978modeling}.
Our regret definitions and
family mixtures use this established prediction--compression framework.
A separate strand asks what a predictor with limited resources can achieve:
\citet{meron2004finite} characterize the regret of finite-state universal
predictors of individual binary sequences as a function of the number of
states, \citet{ingber2006prediction} design deterministic finite-state
predictors non-asymptotically, and \citet{dar2011finite} extend the
finite-memory analysis to continuous sequences. Our comparison of
transformers of increasing depth and training with the unconstrained
Bayesian predictor is the analogous question for a different resource, and
Appendix~\ref{app:why} makes the connection explicit.

\citet{feder2025framework} propose an information-theoretic framework for modern machine learning in which a randomly initialized architecture is a prior over predictors and learning is universal prediction with respect to that prior. The cost of learning a target predictor is its \emph{architecture-based complexity}, the negative log of the prior mass that the architecture places near the target, and the distribution of this cost over targets---the architecture's \emph{complexity spectrum}---is what distinguishes one architecture from another. \citet{feder2026lsa} carry the program out exactly for the layered simplex architecture: its prior on large-alphabet distributions has a computable Bayesian mixture, a computable learning curve and a complexity spectrum that is broad where a symmetric Dirichlet prior is narrow. The present paper uses that architecture in the opposite role. The LSA prior is here the \emph{source} prior, and its row integrals are the building blocks from which informed Bayes is assembled for known and hidden products, shared latent bases and depths, and mixtures over source families; where these combinations no longer have closed forms, we use controlled numerical marginalization and validate the resulting predictive probabilities as described in Appendix~\ref{app:reference-numerics}. The transformer, meanwhile, is the architecture whose prior over predictors is being probed, and informed Bayes is the target whose architecture-based complexity training has to pay. Appendix~\ref{app:twolevels} states this two-level picture.

Thus universal coding, the Bayesian interpretation of in-context learning,
and Markov prediction by induction mechanisms all precede this work. Our
contribution is to join them in a set of source families for which the learned
structure has a quantitative price: informed Bayes predicts the
learning curves, and the same source model predicts transfer, adaptation,
misspecification loss, and the emergence of copying.

\section{Supplemental material for Section~\ref{sec:sources-prediction}: the Bayesian predictors}
\label{app:strategies}

This appendix defines the Bayesian predictors used in the paper. For coins
the posterior prediction has a simple formula. For the LSA-based sources we
first derive the probability of the counts in one row from the integrals of
\citet{feder2026lsa}, and then combine the rows into the probability of a
sequence. We then sum over the unknown encoding of the product source. The
mixture over families is described with its experiment in
Appendix~\ref{app:pooled-comparison}. The frequency-only predictor of
Section~\ref{sec:application-copying} is defined in
Appendix~\ref{app:shared-shuffle-control}.
Table~\ref{tab:sources-at-a-glance} summarizes the sources.

\begin{table}[tbp]
\caption{\textbf{The six sources.} Training can teach what the middle
column lists. A new sequence must reveal what the right column lists.}
\label{tab:sources-at-a-glance}
\centering\small
\begin{tabular}{p{0.22\linewidth}p{0.35\linewidth}p{0.31\linewidth}}
\toprule
Source & Shared by all sequences & Drawn afresh for every sequence\\
\midrule
Coins & the prior on the bias & the bias\\
Known encoding & the two coordinates, the encoding, the prior on the depths & the two depths and the two transition matrices\\
Hidden encoding & the two coordinates & the encoding and the two transition matrices\\
Uniform base & the uniform base law, the prior on the depth & the depth and the $64$ transition rows\\
Shared base & one base law reshaped by every row, the priors on $\alpha$ and on the depth & $\alpha$, the ranks of the tokens, the depth and the rows\\
Pooled families & the three families and their priors & the family and its parameters\\
\bottomrule
\end{tabular}
\end{table}

\subsection{Coins: sequence probability and prediction}
For a prefix $x^n$, let $n_1=\sum_{i=1}^n x_i$ be its number of ones and
$n_0=n-n_1$ its number of zeros. Under a prior $\pi$ on the coin bias, the
Bayesian probability of this particular ordered sequence is
\[
 Q_B(x^n)=\mathbb E_{\theta\sim\pi}\!\left[
 \theta^{n_1}(1-\theta)^{n_0}\right].
\]
For the two priors used here, this gives
\[
 Q_B(x^n)=
 \begin{cases}
 \tfrac12\left[0.2^{n_1}0.8^{n_0}+0.8^{n_1}0.2^{n_0}\right],
     &\text{equal prior mass at }0.2,0.8,\\[3pt]
 \dfrac{n_1!\,n_0!}{(n+1)!},&\text{uniform prior on }[0,1].
 \end{cases}
\]
The probability of a one next is the ratio of the probability of the prefix
followed by a one to that of the prefix itself:
\[
 Q_B(1\mid x^n)=\frac{Q_B(x^n1)}{Q_B(x^n)}.
\]
For the two-point prior, this ratio is the average of the biases $0.2$ and
$0.8$ weighted by their posterior probabilities, which are proportional to
$\theta^{n_1}(1-\theta)^{n_0}$. For the uniform prior, the ratio gives
$(n_1+1)/(n+2)$. In both cases, the counts include the first token,
although its loss is not scored.

\subsection{LSA row prior: row probability and prediction}

Let $m=(m_1,\ldots,m_d)$ be a vector of nonnegative integer counts,
with $k=\sum_bm_b$. A row with depth $L$ and base probability vector
$w=(w_1,\ldots,w_d)$, where $w_b>0$ and $\sum_bw_b=1$, has the representation
\[
 W_{L,b}=\prod_{h=1}^{L}E_{b,h},\qquad
 P_b=\frac{w_bW_{L,b}}{\sum_c w_cW_{L,c}},
 \qquad E_{b,h}\overset{\mathrm{iid}}\sim\operatorname{Exp}(1).
\]
Here $W_{0,b}=1$; the simplex-factor normalizations cancel.
The probability of a particular ordered sequence of row observations with
counts $m$ is
\[
 Z_L(m;w)=\mathbb E_{\{E_{b,h}\}}\!\left[\prod_bP_b^{m_b}\right],
 \qquad Z_L(0;w)=1.
\]
The expectation averages over the independent exponential factors above,
with $L$ and $w$ fixed.
The next-token prediction is $Z_L(m+e_b;w)/Z_L(m;w)$.

Define the transforms
\[
 \phi_{L,r}(s)=\mathbb E_{W_L}\!\left[W_L^r e^{-sW_L}\right],\qquad r\ge0,
\]
where $W_L$ has the common distribution of the $W_{L,b}$ and $r$ is an
integer. Conditioning on the last exponential factor gives
\begin{equation}
 \phi_{0,r}(s)=e^{-s},\qquad
 \phi_{L,r}(s)=\int_0^\infty e^{-x}x^r
                      \phi_{L-1,r}(sx)\,dx\quad(L\ge1).
 \label{eq:source-weight-transform-recursion}
\end{equation}
In particular, $\phi_{1,r}(s)=r!/(1+s)^{r+1}$.

\begin{lemma}
\label{lem:lsa-code-integral}
For $k\ge1$, the row probability is
\begin{equation}
 Z_L(m;w)=\frac1{\Gamma(k)}\int_0^\infty s^{k-1}
       \prod_{b=1}^d w_b^{m_b}\phi_{L,m_b}(sw_b)\,ds.
 \label{eq:weighted-lsa-code}
\end{equation}
For a uniform base, write $Q_L(m)=Z_L(m;\mathbf 1/d)$.
Then
\begin{equation}
 Q_L(m)=\frac1{\Gamma(k)}\int_0^\infty s^{k-1}
                       \prod_{b=1}^d\phi_{L,m_b}(s)\,ds.
 \label{eq:lsa-code-integral}
\end{equation}
\end{lemma}
\begin{proof}
Set $S=\sum_b w_bW_{L,b}$. Substitute
$S^{-k}=\Gamma(k)^{-1}\int_0^\infty s^{k-1}e^{-sS}\,ds$
into the definition of $Z_L$, exchange the nonnegative integral
and expectation, and use independence of the $W_{L,b}$. For
the uniform base, change the integration variable by a factor $d$.
\end{proof}
At depth zero, $Z_0(m;w)=\prod_bw_b^{m_b}$. At depth one with a
uniform base,
\begin{equation}
 Q_1(m)=\frac{(d-1)!\prod_bm_b!}{(k+d-1)!},
 \qquad\frac{Q_1(m+e_b)}{Q_1(m)}=\frac{m_b+1}{k+d}.
 \label{eq:lsa-depth-one-code}
\end{equation}

\subsection{LSA-based sources: sequence probability and prediction}

Shared latent parameters must be averaged after multiplying the
context-row probabilities. Let $M_a(b)$ count transitions from $a$ to
$b$ in the observed sequence, and write $h=(\alpha,L,\sigma)$ for
the shared parameters with their stated prior $\pi(dh)$. Then
\begin{equation}
 Q_B(x^n)=\mathbb E_h\!\left[
 q(x_1)\prod_{a=1}^d Z_L(M_a;q)\right].
 \label{eq:shared-family-evidence}
\end{equation}
The average over the base and depth is outside the product over
contexts. For the uniform-base control, only the common depth is
averaged. For a product chain, each coordinate has a uniform base,
an initial factor $1/d_i$ and its own depth prior, giving
$Q_B(x^n)=Q_{B,1}(u^n)Q_{B,2}(v^n)$. For the shared-base source,
the informative factor $q(x_1)$ remains in the inference even though
the first token is not scored.

For the shared-base source, the posterior is
\begin{equation}
 \pi(dh\mid x^t)=
 \frac{q(x_1)\prod_a Z_L(M_a;q)}{Q_B(x^t)}\,\pi(dh).
 \label{eq:shared-base-posterior}
\end{equation}
Here $q$ is determined by $\alpha$ and $\sigma$. The next-token
prediction is therefore
\begin{equation}
 Q_B(b\mid x^t)=\int
       \frac{Z_L(M_{x_t}+e_b;q)}{Z_L(M_{x_t};q)}
       \,\pi(dh\mid x^t).
 \label{eq:shared-base-predictive}
\end{equation}
All contexts and the initial token inform the joint posterior on
exponent, rank assignment and depth.

\subsection{Product source with hidden encoding: exact marginalization}
\label{app:hidden-factor-bayes}

For a candidate encoding $\sigma$, decode $x^n$ into two coordinate
sequences. Let $m^{\sigma}_{i,a,b}$ count transitions from $a$ to $b$
in coordinate $i$, whose alphabet has size $d_i$.
Integrating its independent uniform-simplex rows gives
\begin{equation}
 Q_\sigma(x^n)=\frac1{d_1d_2}
 \prod_{i=1}^2\prod_{a=1}^{d_i}
 \frac{\Gamma(d_i)\prod_{b=1}^{d_i}
       \Gamma(1+m^{\sigma}_{i,a,b})}
      {\Gamma(d_i+\sum_bm^{\sigma}_{i,a,b})}.
 \label{eq:hidden-factor-evidence}
\end{equation}
Conditional on $\sigma$, the forecast multiplies two add-one
coordinate predictions and maps back to observed labels. Bayes
averages these forecasts with weights proportional to $Q_\sigma(x^n)$.

Relabeling either coordinate leaves its prior unchanged; when $d_1=d_2$,
interchanging the coordinates does too. Consequently we need only one
encoding from each equivalence class, with a uniform prior on the
\[
 K_{d_1,d_2}=\frac{(d_1d_2)!}{d_1!d_2!}
 \begin{cases}1,&d_1\ne d_2,\\[1pt]1/2,&d_1=d_2\end{cases}
\]
classes. The marginal probability is
$Q_B(x^n)=K_{d_1,d_2}^{-1}\sum_{\sigma}Q_\sigma(x^n)$.
This is exact finite summation with analytic row probabilities;
no sampling approximation to the posterior is required. Its cost
grows rapidly: the $2\times8$ source used here has
$K_{2,8}=259\,459\,200$ equivalence classes.

Relative to a reference given the encoding $\sigma_*$ but not the
transition matrices, cumulative excess loss is
$\log_2[Q_{\sigma_*}(x^n)/Q_B(x^n)]\le\log_2K_{d_1,d_2}$.
The first-token law is identical, so excluding that token preserves
the bound. The prior-averaged cost of encoding uncertainty need not attain it.

\section{Why a transformer can approximate an informed Bayes: the mechanism in detail}
\label{app:why}

This appendix expands Section~\ref{sec:why}: first the objective, then
the decomposition of informed Bayes into selection, accumulation and
estimation, what each of the present sources adds to it, why and how the
shared structure is selected, and what the decomposition predicts about
capacity, training and context. The mechanistic statements are hypotheses;
the learning curves of Section~\ref{sec:bayes-evidence} are the evidence we
have for them.

\subsection{The objective, capacity-expressibility and training}
\label{app:objective}

Identity~\eqref{eq:objective} says that the population training loss is
the loss of the informed Bayes plus the expected divergence of the network's
forecast from it. Two things separate a trained network from that target.
The first is its capacity, or expressibility: the architecture must contain a function of $x^t$
that approximates $Q_B(\cdot\mid x^t)$ on the prefixes the prior generates.
The second is training: the population loss is estimated from a finite
number of sequences, and the optimizer must find the approximation. The
regret curves of Figure~\ref{fig:bayes-summary} separate the two only
partly, but both are visible: the hidden-encoding source is matched at
$1024$ tokens by eight layers and not by two, and doubling the training of
the eight-layer models still lowers their loss. In the language of
\citet{jeon2024information}, the weights learn the family and the prior
across sequences; the activations then carry, for each sequence, whatever
statistics of the prefix the posterior needs.

The required capacity question has a classical form. Universal prediction with
a finite number of states asks how much regret a $K$-state machine must pay
against the best constant or Markov predictor of an individual sequence;
\citet{meron2004finite} bound the achievable regret as a function of $K$
for deterministic and randomized machines under the 0--1, square and
self-information losses, \citet{ingber2006prediction} give non-asymptotic
constructions with deterministic machines, and \citet{dar2011finite} extend
the analysis to continuous sequences, where regret against the empirical
mean decays as a power of $K$. In each case a resource-limited predictor
pays a regret that vanishes only as the resource grows, and the rate is a
property of the class. The transformer gaps of Section~\ref{sec:bayes-evidence}
are the same object for a different resource: depth and width, and the
amount of training, in place of states. The ordering of sources by the
capacity they need (Appendix~\ref{sec:capacity}) is a first, qualitative
version of such a rate for these classes.

The reason informed Bayes is the right target, and not merely a
convenient one, is its regret guarantee. For a smooth $k$-parameter family
with a positive prior, the cumulative regret of $Q_B$ against every $\theta$
is about $\tfrac{k}{2}\log_2 t$ plus a constant that depends on the prior
mass near $\theta$ \citep{merhav1998universal}; for a finite mixture over
$m$ hypotheses it is at most $\log_2 m$ in total. A predictor that ignores
part of the structure pays the regret of a larger family, and one that
assumes a wrong family pays a loss that does not vanish with $t$
(Section~\ref{sec:application-mismatch}). The structure is therefore not
optional for a network trained on the family: the objective rewards exactly
those computations that use it.

\subsection{Two universalities: the source family and the architecture}\label{app:twolevels}

The regret guarantee just stated concerns the inner problem, prediction within one sequence, where the class is the source family and the price of not knowing $\theta$ is paid in tokens. A trained transformer faces a second, outer problem, and the framework of \citet{feder2025framework} gives it the same form. A randomly initialized architecture defines a prior $\rho$ over the predictors it can realize; training on sequences from the family is universal prediction with respect to $\rho$; and the price of reaching a target predictor $Q$ is its architecture-based complexity, $-\log_2\rho(\text{predictors within }\epsilon\text{ of }Q)$, paid in training sequences rather than in tokens. Here the target is informed Bayes itself. A network trained on $N$ sequences of length $n$ therefore carries, at test time, a residual expected loss of the order of the architecture-based complexity of $Q_B$ divided by $Nn$, on top of the inner regret $\tfrac{k}{2}\log_2 t$ that $Q_B$ pays against $\theta$. The decomposition of \citet{jeon2024information} into across-task and within-task terms is this two-level structure stated for Bayes-optimal meta-learning; the framework adds that the across-task term is determined by how much of the architecture's prior mass lies near $Q_B$.

Three consequences connect this picture to the experiments. First, the gap $\Delta(t)$ of \eqref{eq:bayes-optimality} measures the outer cost that training has not yet paid, at each inner time $t$; a network with enough training and capacity drives it to zero at every $t$, and the shared-base and hidden-encoding curves of Figure~\ref{fig:bayes-summary} show it doing so. Second, the ordering of sources by required depth is an ordering of the targets $Q_B$ by their architecture-based complexity under the transformer prior. The decomposition $Q_B=F\circ S$ of Appendix~\ref{sec:decomposition} says why: a target whose statistic $S$ is one attention head's worth of selection and counting, and whose estimator $F$ is a smooth map of a few numbers, sits in a region of predictor space that shallow transformers reach with high prior mass; a target that needs a partition inferred from a count matrix before $F$ can be applied sits where only deep compositions reach, and the prior mass there is small until the depth is available. Third, the same framework explains a transformer's preference among predictors that fit the training data equally well: it prefers those of low architecture-based complexity, so when the family leaves a choice open, as the pooled and depth posteriors do, the network's tie-breaking is the architecture's, not the data's; Appendix~\ref{sec:selection} discusses where this can and cannot be distinguished from the Bayesian Occam factor of the family.

The training-size data already in hand give a first test of the outer cost. If the architecture can represent $Q_B$, the residual gap should fall roughly as the inverse of the number of training sequences; if it cannot, the gap should stall at a floor set by capacity. Both regimes appear. For the hidden encoding, the two-layer gap at $t=1024$ falls only from $0.105$ to $0.091$ bits when training grows four-fold from $6.4$ to $25.6$ million sequences (Table~\ref{tab:hidden-factor-scaling}), a floor; the eight-layer gap falls from $0.0057$ to $0.0024$ bits when training doubles from $12.8$ to $25.6$ million, close to the inverse-proportional rate. For the shared-base source the four-layer gap at $(L,t)=(0,1024)$ changes by less than $0.0002$ bits under the same four-fold increase while the eight-layer gap is still falling (Appendix~\ref{app:shared-base-comparison}). The pattern is what the framework predicts: a target of low architecture-based complexity under the deeper prior is reached at the universal rate, and one of high complexity under the shallow prior is not reached at all. A dedicated scaling study, with several training sizes per architecture, would turn this into a measured exponent.

Two remarks on scope. The complexity spectrum of the layered simplex architecture is computable \citep{feder2026lsa}, and that is why the LSA prior serves here as the source prior: informed Bayes can be evaluated. The complexity spectrum of a transformer is not computable in the same sense, and the present paper does not attempt it; what it measures is, for six computable targets, how much depth, width and training the architecture needs to reach them, which is the empirical shadow of that spectrum. The meta-learning experiments of \citet{graumoya2024universal} are the same measurement for a different target, the context-tree-weighting mixture and, aspirationally, Solomonoff's mixture, and their finding that transformers and LSTMs reach the same loss says that both architectures place adequate prior mass near that target at the scale they used; which of the two does so with the smaller complexity, and hence with fewer sequences, is the question the framework poses and their length-generalization results begin to answer.

\subsection{Selection, accumulation, estimation}
\label{sec:decomposition}

For every source studied here the Bayesian prediction can be written as
a composition $Q_B(b\mid x^t)=F\big(S(x^t)\big)_b$, where $S$ is a set of
statistics of the prefix and $F$ maps these statistics to a distribution
over the alphabet. The simplest case, and the one in which the imitation has
been observed mechanistically, is a Markov chain whose rows are drawn
independently from the uniform simplex \citep{edelman2024evolution}. There
$S$ consists of the successor counts in the current context,
$m_a(b)=\#\{s<t:x_s=a,\,x_{s+1}=b\}$ with $a=x_t$, and $F$ is the add-one
rule $(m_a(b)+1)/(\sum_c m_a(c)+d)$; equation~\eqref{eq:lsa-depth-one-code}
gives the LSA generalization at depth one. A transformer realizes this in two
steps. An attention head whose query is the current token and whose keys are
the past tokens attends to the positions $s$ with $x_s=x_t$, and averages the
one-hot successors at those positions; the output is the empirical successor
distribution in the current context, and the number of matches is
recoverable from the attention mass on the start token. This is the
statistical induction head. The feed-forward blocks and the unembedding then
have to represent $F$. Attention thus supplies two of the three ingredients
for free---\emph{selection} of the positions that share the current
situation, and \emph{accumulation} of counts over them---and the third,
the \emph{estimator} $F$, is learned by the feed-forward network. The estimator is what the prior
fixes: the same counts lead to different predictions under different priors,
and the coin experiment of Section~\ref{sec:bayes-coins} is the case with no
selection at all, where the two-point-prior model and the uniform-prior model
must and do learn the two different maps
of~\eqref{eq:mismatch-coin-predictive}.

Three properties of this decomposition carry over to the harder
sources. Selection is content-addressed and need not be consecutive, so the
same mechanism serves a context defined by the current token, by one of its
coordinates, or by the whole prefix. Accumulation is exact and costs no
parameters. Estimation has to be learned on the count configurations that
the prior actually produces; the network is never trained on configurations
it will not see, and it is trained most on the ones that occur most, which
is why the residual gap depends on $t$ in the way described in
Appendix~\ref{sec:capacity}.

\subsection{What is different about the present sources}
\label{sec:different}

In the independent-row chain, and in the coins, regression tasks and
$n$-gram families of earlier work, the statistics $S$ are local: the
prediction in context $a$ depends on the prefix only through the counts in
row $a$, and $F$ acts row by row. The sources of
Section~\ref{sec:sources-prediction} break this locality in three distinct
ways, and each way has a distinct signature.

\paragraph{Global statistics: the shared-base source.}
Under the shared-base prior the prediction in row $a$ is, by
\eqref{eq:shared-base-predictive}, the posterior average of
$Z_L(M_a+e_b;q)/Z_L(M_a;q)$ over the base law $q$ (that is, over the exponent
$\alpha$ and the rank permutation $\sigma$) and the common depth $L$. The
row's own counts $M_a$ enter as before, but $q$ and $L$ are inferred from
all rows together: the rank assignment mainly from the total token
frequencies $N_b=\sum_a M_a(b)$, which order the labels; the exponent from
how steeply those totals fall; and the depth from how concentrated every row
looks, i.e.\ from the count profiles of the rows. So the sufficient
statistics contain a global part---the totals and the profiles---that
requires attention over \emph{all} past positions, not only the matching
ones, and the estimator has to combine the global part with the local
counts. The uniform-base control removes the global part: $q$ is known, and
the only shared quantity is the depth, which matters little for prediction.
The unseen-context test of Section~\ref{sec:application-transfer} isolates
the global part experimentally. When the current token appears for the first
time, the local counts are empty and the prediction is the posterior mean of
the base law, $\int q\,\pi(dq\mid x^t)$, which is a function of the global
statistics alone. The transformer's gain over uniform there, $0.793$ bits
against Bayes's $0.799$ at $L=1$, $t=64$, and its near-zero gain on the
uniform-base control where Bayes predicts none, is direct evidence that it
computes the global statistics and uses them where the family says they are
informative, and only there.

\paragraph{A hidden partition: the product source.}
With a known encoding, the product structure changes selection but not
estimation: the head for coordinate $u$ must attend to the positions whose
token shares the current token's $u$-coordinate, which is a fixed linear
function of the token embedding, and the two coordinate predictions are
multiplied. Two heads and one layer suffice in principle, and the two-layer
model is matched to Bayes at $1024$ tokens. With a hidden encoding the
partition of the $16$ labels into $2\times8$ coordinates is itself a latent
variable of the sequence. Bayes handles it by summing the factorized
evidence $Q_\sigma(x^n)$ over all $259\,459\,200$ equivalence classes of
encodings (Appendix~\ref{app:hidden-factor-bayes}); no network can enumerate
them, so a transformer has to infer the partition from statistics of the
prefix. Labels that share a coordinate have successor distributions that
agree on that coordinate's marginal, so the partition is readable from the
$16\times16$ count matrix, but only after the counts are formed: one layer
to accumulate per-label successor statistics, further layers to compare them
and form a soft partition, and then the factorized estimate. This is, we
suggest, why the hidden-encoding source needed eight layers and $25.6$
million sequences where the known-encoding source needed two layers and
$6.4$ million, and why the residual gap concentrates at intermediate
contexts: at $256$ tokens the gap is $0.022$--$0.025$ bits while at $1024$
it is $0.002$, and $256$ tokens is where the posterior over encodings is
neither flat nor concentrated, so that the soft sum over encodings---the
hardest part of $F$ to imitate---matters most
(Figure~\ref{fig:hidden-factor-scaling},
Table~\ref{tab:hidden-factor-scaling}).

\paragraph{Selecting the family: pooled training.}
Under pooled training the family $f$ is one more latent variable. The
Bayesian prediction is the
mixture~\eqref{eq:pooled-family-weights}--\eqref{eq:pooled-prediction} of
Appendix~\ref{app:pooled-comparison}, whose weights are a softmax of the
cumulative log-likelihoods of the three family predictors. Family selection
is therefore the same computation as parameter selection one level up: each
family's evidence is a sum over positions of that family's per-position
log-probabilities, a quantity attention can accumulate, and the weights are
a softmax of those sums. A network trained on the pooled prior has to carry
three family-specific predictions and their running evidences at once,
which requires the network to combine within-family inference with an
additional weighting step. At four observed tokens, increasing the model
from two to four layers reduces the gap to the Bayesian mixture from
$0.0197$ to $0.0052$ bits in run~1 and from $0.0190$ to $0.0025$ bits in
run~2, at the same training budget.

\subsection{Why and how the structure is selected}
\label{sec:selection}

Why a trained network uses the shared structure rather than ignoring it
follows from Section~\ref{sec:why}: under the training prior, ignoring
the structure costs regret at every context length, and the objective is the
expected loss over exactly that prior. The size of the incentive is the
difference between the regret of the family and the regret of the larger
family the network would otherwise be learning. For the shared-base source
at depth zero, the family has the $d-1$ parameters of the base law and the
tokens are independent; a predictor that treats the $64$ rows as independent
would have to learn $64\times63$ parameters from the same tokens, and its
regret would be far larger at every tested $t$. The pooled setting makes the
incentive explicit: not knowing the family costs Bayes at most $\log_2 3$
bits cumulatively (Appendix~\ref{app:application-family}), whereas
predicting with the wrong family costs $R_\infty$ bits on every token
(Section~\ref{sec:application-mismatch}). A network that did not select
would pay the second cost; one that selects as Bayes does pays the first.

How the structure is selected is the posterior over the discrete latent
variables---the family, the depth, the encoding---which is a softmax of
accumulated log-evidence, and the posterior over the continuous parameters,
which the estimator $F$ represents implicitly. Two consequences follow.
First, selection is soft and stays soft: the posterior over families does
not always concentrate, because the shared-base and uniform-base priors both
place positive mass on every positive transition matrix and remain ambiguous
with unlimited data (inequality~\eqref{eq:family-cost-limit}). A network
imitating $Q_B$ must remain equally undecided, and the pooled comparisons at
longer contexts, once the Bayesian references are complete, test exactly
this. Second, the preference for the simpler structure that transformers
display when the evidence is ambiguous \citep{deora2025occam} is, in our
setting, the prior's preference and not the architecture's: when one family
is nested in another, or a lower depth is nested in a higher one, the
evidence of the smaller family integrates over fewer parameters, so its
posterior weight is larger at short contexts and yields only as the counts
accumulate. This is the Bayesian universality or Occam factor. A network that imitates $Q_B$
inherits it, and departs from it only where it departs from $Q_B$; the depth
posterior of the shared-base source, which selects the lower depth until the
row profiles rule it out, is the place to test this directly.
Where the network does depart from $Q_B$, the architecture's own preference of Appendix~\ref{app:twolevels}---for predictors of low architecture-based complexity---is the natural candidate for the direction of the departure, and a test that separates the two universality factors would compare the network's family weights with the Bayesian ones on prefixes where the family evidence is balanced.

\subsection{What this predicts about capacity, training and context}
\label{sec:capacity}

The decomposition turns the open questions of the introduction into
concrete expectations, which the present experiments partly test and future
ones can test further.

\emph{Capacity} which allows the expressibility of the approximated target is set by the depth of the composition $S\to F$ and by
the size of $S$. Local counts need one layer of selection and one map; a
hidden partition needs layers to form statistics, compare them and select;
global statistics need width to carry $d$ totals and row profiles beside the
local counts. The sources here order themselves accordingly: coins and the
known-encoding product are already well approximated by two layers; four
layers improve both the known-encoding and uniform-base comparisons; and
eight layers bring the hidden-encoding and shared-base models close to
informed Bayes. Four layers also improve the pooled model, which must
carry three family-specific predictions.

\emph{Training} is set by how many distinct count configurations the
estimator has to be learned on. Configurations that the prior produces
rarely are learned last; for the shared-base source these are the deep,
sparse rows at long contexts, and it is at $t=1024$ that doubling the
eight-layer training still helps at every depth
(Figure~\ref{fig:shared-base-conditional}).

\emph{Context} enters through the nonlinearity of $F$. At short contexts
the posterior is close to the prior and $F$ is nearly constant; at long
contexts the posterior is concentrated and $F$ is nearly the empirical
frequency; in between, $F$ is most curved and the network's approximation is
weakest. This predicts that the residual gap peaks at intermediate $t$, and
this occurs in several comparisons: at $t=32$ for the two-layer
depth-$(1,1)$ product model
(Figure~\ref{fig:product-conditional-gaps}), at $t=128$ for the two-layer
uniform-base $L=0$ control, and at $t=256$ for the hidden encoding
(Section~\ref{sec:bayes-hidden-factor}). The Bayesian curve says where the
hard part of each source lies, and the transformer's gap sits there.

These expectations are consistent with the comparisons reported here;
they are not established by them. A direct test of the mechanism would read
the attention patterns and the count-like activations against the statistics
that $Q_B$ needs, source by source.

\section{Details for Section~\ref{sec:bayes-evidence}: how close transformers get to informed Bayes}
\label{app:bayes-evidence-details}
\label{app:context-accuracy-audit}

This appendix follows the same progression as
Section~\ref{sec:bayes-evidence} and Figure~\ref{fig:bayes-summary}.
Appendix~\ref{app:transformer-details} first gives the common experimental
and numerical protocol. The remaining subsections then treat the six
inference problems in panel order: a source prior, a known factorization,
a hidden factorization, fixed shared structure, inferred shared structure,
and an unknown source family.

\subsection{Common experimental and numerical protocol}
\subsubsection{Models, training and test sets}
\label{app:transformer-details}

All models are decoder-only transformers with learned token and positional
embeddings, pre-layer-normalized causal-attention and feed-forward blocks,
residual connections, a final layer normalization and linear output. The
feed-forward width is four times the model width, with GELU and no dropout.
Table~\ref{tab:training-settings} specifies every model used in Section~\ref{sec:bayes-evidence}.
Run~1 uses random seed~0 and run~2 uses seed~1. Within a run, all tests of a
model at a given training length use the same saved weights.

\begin{table}[tbp]
\centering\small
\setlength{\tabcolsep}{3.5pt}
\begin{tabular}{lrrrrrrr}
\toprule
Source and model & Layers & Width & Heads & Parameters & Updates & Seq.\ ($10^6$) & Runs\\
\midrule
Coins, each prior & 2 & 64 & 4 & $116736$ & $20000$ & $1.28$ & 1\\
Product, known $8\times8$ encoding & 2 & 128 & 4 & $544256$ & $100000$ & $6.4$ & 1\\
Product, known, four-layer & 4 & 256 & 8 & $3454464$ & $100000$ & $6.4$ & 2\\
Product, hidden, two-layer & 2 & 128 & 4 & $531968$ & $100000$ & $6.4$ & 2\\
Product, hidden, continued & 2 & 128 & 4 & $531968$ & $400000$ & $25.6$ & 2\\
Product, hidden, four-layer & 4 & 256 & 8 & $3429888$ & $100000$ & $6.4$ & 2\\
Product, hidden, four-layer continued & 4 & 256 & 8 & $3429888$ & $400000$ & $25.6$ & 2\\
Product, hidden, eight-layer & 8 & 512 & 8 & $25760768$ & $400000$ & $25.6$ & 2\\
Uniform-base LSA, $d=64$ & 2 & 128 & 4 & $544256$ & $100000$ & $6.4$ & 1\\
Uniform-base LSA, four-layer & 4 & 256 & 8 & $3454464$ & $100000$ & $6.4$ & 2\\
Shared-base LSA, two-layer & 2 & 128 & 4 & $544256$ & $100000$ & $6.4$ & 2\\
Shared-base LSA, four-layer & 4 & 256 & 8 & $3454464$ & $100000$ & $6.4$ & 2\\
Shared-base LSA, four-layer, shuffled & 4 & 256 & 8 & $3454464$ & $100000$ & $6.4$ & 2\\
Shared-base LSA, four-layer continued & 4 & 256 & 8 & $3454464$ & $400000$ & $25.6$ & 2\\
Shared-base LSA, eight-layer & 8 & 512 & 8 & $25809920$ & $400000$ & $25.6$ & 2\\
Pooled, three families & 2 & 128 & 4 & $544256$ & $300000$ & $19.2$ & 2\\
Pooled, four-layer & 4 & 256 & 8 & $3454464$ & $300000$ & $19.2$ & 2\\
\bottomrule
\end{tabular}
\caption{\textbf{Training settings for the reported models.}
Contexts are $256$ tokens for coins and $1024$ otherwise; sequences
contain one additional target token. Every update uses $64$ fresh
sequences. Continued models retain their checkpoint parameters and optimizer state;
other model sizes are trained from fresh initialization.}
\label{tab:training-settings}
\end{table}

Training minimizes mean next-token cross-entropy in single precision.
AdamW uses $(\beta_1,\beta_2)=(0.9,0.95)$, matrix-parameter weight decay
$0.1$ (none for biases or layer-normalization parameters), and gradient
norm clipping at $1$. The learning rate rises linearly to $3\cdot10^{-4}$
then follows a cosine schedule to $3\cdot10^{-5}$. Warmup lasts $500$
updates for coins, $3000$ for the pooled model, and $1000$ for the other
models. Continuations instead retain the final learning rate for their
additional $300000$ updates. The eight-layer models use a cosine schedule
over $400000$ updates; their final evaluations therefore change
size, amount of training and learning-rate schedule relative to the earlier
$100000$-update models. All next-token positions
enter the training loss; each sequence receives newly drawn source
parameters under its full training prior. The pooled training schedule
balances the three families, giving $6.4$ million sequences per family
and $19.2$ million in total. Training, monitoring and test data come from
separate random streams. The code repository lists the configuration
and the command for every model and every figure. In the product source with
known encoding, the generator also relabels the values of each coordinate at
random for every sequence. This does not change the distribution of the
source, because the rows are independent and exchangeable.

The four-layer known-product and uniform-base models use the same source
priors, training budgets and learning-rate schedules as their two-layer
counterparts, with independently drawn training sequences. Both larger
runs are compared with the same two-layer checkpoint on the same test
sequences. For pooled training, the two sizes additionally receive the
same indexed training sequences within each run. All comparisons use
the final checkpoint; no run is selected by its test performance.

\begin{table}[tbp]
\centering\small
\begin{tabular}{llrr}
\toprule
Source & Fixed test condition & Sequences & Length\\
\midrule
Two-point coins & $\theta=0.2$; $0.8$ & $2076$; $2020$ & $257$\\
Uniform-prior coins & Each $\theta=0.1,0.3,0.5,0.7,0.9$ & $4096$ & $257$\\
Product $8\times8$ & Each $(1,1),(1,8),(4,4),(8,8)$ & $256$ & $1025$\\
Hidden encoding & Full prior & $3840$ & $1025$\\
Uniform-base LSA & $L=0,1,4,8$, respectively & $112,101,115,119$ & $1025$\\
Shared-base LSA & $\alpha=1$; each $L=0,1,4,8$ & $256$ & $1025$\\
Pooled families & Each family's full prior & $1024$ & $1025$\\
\bottomrule
\end{tabular}
\caption{\textbf{Test sets for the main comparisons.}
The parameters that define a test setting are never shown to the predictors.
The uniform-base groups are parts of one full-prior test set of $1024$
sequences. The hidden-encoding test set leaves out the $256$ sequences that
we used while developing the experiment.}
\label{tab:conditional-cohorts}
\end{table}

\subsubsection{Scoring and uncertainty}
\label{app:prediction-metrics}
A sequence of length $N+1$ supplies $N$ scored predictions: after
observing $X^t$, the predictor scores $X_{t+1}$, for $t=1,\ldots,N$.
The first token informs inference but is not scored. Conditional expected
next-token loss is
\begin{equation}
 h_t(Q;\theta)=-\sum_b P_\theta(b\mid X^t)\log_2 Q(b\mid X^t).
 \label{eq:causal-expected-loss}
\end{equation}
We suppress $\theta$ when the realized source is understood. For a
Markov source the scoring law is its current true transition row.
Subtracting that row's entropy and averaging gives
\eqref{eq:instantaneous-regret}. Source parameters are used only for
scoring and grouping, never supplied to the predictors. Histories follow
the specified initial law; stationarity is not assumed.

With the first token supplied, cumulative loss and regret through token
$n$ are
\begin{equation}
 \mathcal L_M(x^n)=-\sum_{t=1}^{n-1}\log_2 Q_M(x_{t+1}\mid x^t),
 \qquad R_M(n)=\sum_{t=1}^{n-1}r_M(t).
 \label{eq:cumulative-regret}
\end{equation}
In particular, $\mathcal L_B(x^n)=-\log_2[Q_B(x^n)/Q_B(x_1)]$.
For independent sequences $j=1,\ldots,N$, uncertainty in a statistic
$Y_j$ uses
\[
 \overline Y=N^{-1}\sum_jY_j,\qquad
 \operatorname{SE}(\overline Y)=
 \left[\frac{\sum_j(Y_j-\overline Y)^2}{N(N-1)}\right]^{1/2}.
\]
Model differences and cumulative sums are formed within each sequence
before aggregation. Tokens are not independent replicates; training-run
variation is reported separately. Unless stated otherwise, numerical
values give one standard error across sequences, not numerical
integration error or simultaneous coverage across times.

Averaging over the full training prior $\pi$ gives the population identity
\[
 r_M(t)-r_B(t)=\mathbb E_\pi D_{\mathrm{KL},2}
 \bigl(Q_B(\cdot\mid X^t)\Vert Q_M(\cdot\mid X^t)\bigr)\ge0.
\]
An estimated gap can nevertheless be negative through sampling error.
If evaluation fixes parameters or selects a different source distribution
while Bayes retains its training prior, the identity and ordering need
not hold even in the population.

\subsubsection{Bayesian evaluation and numerical precision}
\label{app:reference-numerics}

\paragraph{Rows, coins and known-encoding chains.}
We tabulate the transforms in
\eqref{eq:source-weight-transform-recursion} once and reuse them. The row
integral is computed on a grid in $u=\ln s$. Products are evaluated as sums
of logarithms, and a base weight $w_b$ enters as the shift $u+\ln w_b$. For
sources with a common depth, the row probabilities are multiplied first, and
the average over the depth is taken afterwards. Independent coordinates are
then combined. We check the results in four ways: against the closed forms
at depth zero and depth one, against the product factorization, by comparing
sequential with batch evaluation, and by refining the grid. In $198$ row
checks with up to $1024$ observed tokens, a wider integration range and half
the grid spacing changed the log probabilities by less than
$9\cdot10^{-12}$ bits, and the closed-form cases agreed within
$6\cdot10^{-11}$ bits. These checks cover the tested count vectors, not
every possible history. For coins, closed-form count formulas and finite
binomial sums check the Bayesian predictions and regrets. A separate script
recomputes every reported number from the saved predictions. It checks the
identity of source and model, that every prediction uses only earlier
tokens, that the probabilities sum to one, and the averages and paired
standard errors. We never rescale probabilities to make results agree.

\paragraph{Shared-base inference.}
At $t=1,4$, we sum over the ranks of observed labels and integrate over
exponent and shared depth. There are $15$ repetition patterns at four
tokens. Small-alphabet enumeration checks this calculation; refinement
changes the reported regrets by less than $2\cdot10^{-14}$ bits.
At $t=16,64$, posterior sampling covers the unknown rank assignment
and exponent while integrating over the common depth and rows. Every
sequence has two independently seeded calculations at $16$ and $32$
exponent nodes, each with four chains and $16384$ warmup moves.
Retained draws are extended from $4096$ to as many as $65536$ per chain
at $t=16$, and $262144$ at $t=64$. A separate script checks chain
agreement, effective sample sizes, predictive moments and refinement
of both integration rules, requiring split $\widehat R<1.05$ for all
monitored coordinates, probability effective sample sizes at least $200$,
maximum probability Monte Carlo standard error at most $0.005$, and
raw conditional and averaged predictive normalization errors at most $10^{-7}$.

Numerical error is assessed separately from source-to-source variation.
For a test set of $N$ sequences, the numerical variance of its mean is
the sum of individual numerical variances divided by $N^2$. The
numerical error bound combines twice the numerical standard error and twice the
estimated bias from taking a logarithm of a sampled probability. It
must be below $0.01/3$ bits and below one quarter of the sequence
standard error after subtracting estimated numerical variance.
All $256$ sequences enter every reported depth-specific mean.
The largest accepted numerical error bounds are $0.0023$ bits at $t=16$ and
$0.0017$ at $t=64$; the largest time-$64$ change between integration
orders is $0.00050$ bits. All $2048$ individual time-$64$ references
(two orders for each of four test sets of $256$ sequences) pass the checks.

\subsection{A source prior: coins}
\label{app:coins-comparison}

The other coin comparisons confirm that the transformer learns the training
prior. Figure~\ref{fig:bayes-summary} shows $t=1,2,4,\ldots,256$. For the
two-point prior, the test sequences at $\theta=0.2$ and at $\theta=0.8$ are
the two parts of one full-prior test set. For the uniform prior we drew new
sequences at each displayed bias. Closed-form binomial sums confirm the
means of informed Bayes.

For $4096$ independent full-prior sequences per coin model, cumulative
gap over $256$ predictions is $0.0017\pm0.0013$ bits for the
two-point prior and $-0.0002\pm0.0014$ for the uniform prior.
The sequence-averaged predictive divergences are respectively
$(9.47\pm0.20)\cdot10^{-6}$ and $(1.20\pm0.04)\cdot10^{-5}$ bits.
We reconstruct binary probabilities from the saved probability of the
observed token, verify them against full vectors and replay fixed sequences
against the checkpoints, avoiding lower-precision vectors when estimating
these small divergences.

\subsection{A known encoding: product source}
\label{app:known-product-comparison}

Four layers bring the transformer closer to informed Bayes at the same
training budget. Across the four depth pairs, the gaps after $1024$
observed tokens are $0.0012$--$0.0031$ bits in run~1 and
$0.0026$--$0.0056$ bits in run~2, compared with $0.004$--$0.009$ bits
for the two-layer model. The largest gaps across the evaluated contexts
are $0.044\pm0.005$ and $0.049\pm0.006$ bits for the four-layer runs,
both at $(L_1,L_2)=(1,1)$ after $16$ tokens; the two-layer maximum is
$0.124\pm0.009$ bits after $32$ tokens. The comparison uses the same
$256$ sequences at each depth pair and $t=1,2,4,\ldots,1024$.
Figure~\ref{fig:known-structure-scaling} shows both larger runs.

\begin{figure}[tbp]
\centering
\includegraphics[width=\linewidth]{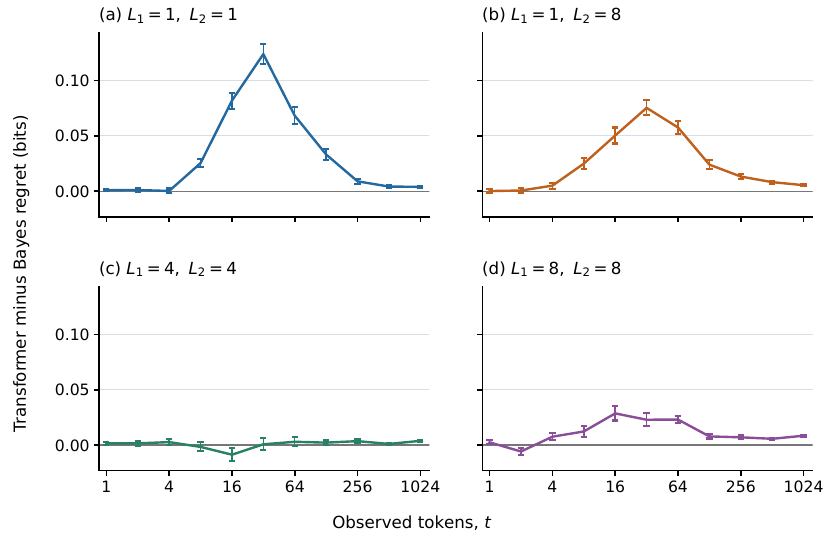}
\caption{\textbf{A known factorization: the two-layer gap to informed Bayes.}
Each panel uses $256$ independent sequences at its fixed coordinate-depth
pair, matching panel (b) of Figure~\ref{fig:bayes-summary}. Positive values
mean larger expected next-token loss for the transformer. Bars give one
standard error of the paired differences across sequences. The common
vertical scale makes the different transient gaps visible.}
\label{fig:product-conditional-gaps}
\end{figure}

The full-distribution comparison also measures the two-layer model's
agreement. Figure~\ref{fig:bayes-evidence}(a) reports
$D_t=D_{\mathrm{KL},2}(Q_B(\cdot\mid X^t)\Vert
Q_{\rm TF}(\cdot\mid X^t))$ at every evaluated context. On $128$
independent full-prior sequences, its accumulated transient contributes
a cumulative gap of $7.23\pm0.50$ bits over $1024$ predictions.

\subsubsection{Predicting the required context}
\label{app:application-context}

For target regret $\epsilon$ and tested times $\mathcal T$, define
\begin{equation}
 n_M(\epsilon;\mathcal T)=\min\{t\in\mathcal T:r_M(t)\le\epsilon\},
 \label{eq:required-context}
\end{equation}
with nonattainment if the set is empty. Close curves can have different
crossings near a flat region. We therefore fixed targets
$\epsilon=0.1,0.2,0.5$ bits and a refined grid before independent
confirmation. Starting with $\{1,2,4,\ldots,1024\}$, we added
$\lceil n_B2^{j/4}\rceil$, $j=-4,\ldots,4$, around each development
Bayesian crossing $n_B$, clipped to $[1,1024]$. The resulting grid has
$26$ times. At each of four depth pairs, $1024$ fresh sequences of length
$1025$ were drawn; all Bayesian forecasts were saved before evaluating
the trained transformer. Neither predictor receives the source depths.
No interpolation, monotonic smoothing or deletion of unreached outcomes
is used.

\begin{table}[tbp]
\centering\small
\begin{tabular}{c|cc|cc|cc}
\toprule
& \multicolumn{2}{c|}{$\epsilon=0.1$}
& \multicolumn{2}{c|}{$\epsilon=0.2$}
& \multicolumn{2}{c}{$\epsilon=0.5$}\\
Depths & Inf.\ Bayes & Transformer & Inf.\ Bayes & Transformer & Inf.\ Bayes & Transformer\\
\midrule
$(1,1)$ & 725 & 862 & 363 & 363 & 91 & 108\\
$(1,8)$ & 609 & 725 & 256 & 305 & 77 & 91\\
$(4,4)$ & 725 & 725 & 305 & 305 & 91 & 91\\
$(8,8)$ & 512 & 512 & 216 & 216 & 64 & 77\\
\bottomrule
\end{tabular}
\caption{\textbf{First tested context reaching each regret target.}
Six pairs coincide; six differ by one grid step, a ratio of about
$1.18$--$1.20$. Simultaneous uncertainty sets overlap for every pair;
the two predictors therefore require the same context within the resolution
of this experiment.}
\label{tab:context-requirements}
\end{table}

For uncertainty, $4096$ whole-sequence bootstrap resamples retain all
times and both predictors together within each depth condition, while
conditions are resampled independently. The maximum standardized mean
deviation gives a simultaneous $95\%$ band with multiplier $3.60$.
A crossing time is compatible if its lower bound reaches the target
and no earlier upper bound forces a crossing; nonattainment is retained
if every upper bound remains above the target. Both sets are finite
and overlap in all twelve comparisons. This concerns one fixed model.
Exact depth-one checks, wider integration, causal token-index tests
and independent aggregation of all $4096$ sequences pass.

\subsection{Discovering a factorization: hidden encoding}
\label{app:hidden-factor-experiment}

The $2\times8$ source makes the factorization latent while retaining an
exact informed Bayes. Every draw has fresh transition matrices
and a fresh encoding. At $t=1,4,16,64,256,1024$, each transformer and informed
Bayes receive the same tokens; informed Bayes integrates over the unknown
encoding. An oracle that is told the encoding provides a separate measure of
the cost of discovering the factorization.

Scaling brings the transformer to the informed Bayesian curve. At $t=1024$,
the two-layer models have regrets $0.1430$ and $0.1335$ bits in the
two runs, compared with $0.03781$ bits for informed Bayes. The eight-layer
models trained on $25.6$ million sequences reach $0.0402$ and $0.0397$ bits,
leaving gaps of $0.00240\pm0.00016$ and $0.00184\pm0.00010$ bits.
Doubling their training from $12.8$ to $25.6$ million sequences reduces the
final gap by a further $0.0033$ bits in each run.

The intermediate models show the same progression. Increasing the two-layer
models' training from $6.4$ to $25.6$ million sequences reduces their endpoint
regret by $0.0147\pm0.0005$ and $0.0136\pm0.0005$ bits. At $6.4$ million
sequences, moving from two to four layers reduces it by
$0.0339\pm0.0007$ and $0.0583\pm0.0010$ bits. The final four-layer regrets
are $0.1067$ and $0.0563$ bits, establishing the contribution of size
before the eight-layer models nearly close the remaining gap.

The oracle comparison separates this scaling result from the intrinsic cost
of the hidden encoding. At $t=1024$, regret is $0.03781$ bits when the encoding
is hidden and $0.03727$ when it is revealed; integrating this difference over
$1024$ predictions gives $27.50\pm0.03$ bits. The two-layer transformers have
additional cumulative gaps of $108.38\pm0.63$ and $105.35\pm0.62$ bits,
which size and training reduce to $9.57\pm0.10$ and $8.49\pm0.09$ bits.

All reported comparisons use the same $3840$ evaluation sequences and exclude
the $256$ development sequences. A separate script checks the test set,
the checkpoint identities, the informed-Bayes values, the model scores and
the paired statistics. Table~\ref{tab:hidden-factor-scaling} gives the endpoint values;
Figure~\ref{fig:hidden-factor-scaling} gives both runs at every context.

\begin{table}[tbp]
\centering
\small
\setlength{\tabcolsep}{4pt}
\begin{tabular}{lcrrr}
\toprule
Predictor & Run & Next-token regret & Gap to informed Bayes & Cumulative gap\\
\midrule
Informed Bayes & -- & $0.0378\pm0.0004$ & $0$ & $0$\\
2 layers, $6.4$M & 1 & $0.1430\pm0.0014$ & $0.1052\pm0.0013$ & $108.38\pm0.63$\\
2 layers, $6.4$M & 2 & $0.1335\pm0.0013$ & $0.0957\pm0.0012$ & $105.35\pm0.62$\\
2 layers, $25.6$M & 1 & $0.1283\pm0.0013$ & $0.0905\pm0.0012$ & $99.81\pm0.59$\\
2 layers, $25.6$M & 2 & $0.1199\pm0.0012$ & $0.0821\pm0.0011$ & $98.63\pm0.58$\\
4 layers, $6.4$M & 1 & $0.1091\pm0.0011$ & $0.0713\pm0.0010$ & $94.07\pm0.58$\\
4 layers, $6.4$M & 2 & $0.0751\pm0.0008$ & $0.0373\pm0.0007$ & $55.99\pm0.41$\\
4 layers, $25.6$M & 1 & $0.1067\pm0.0011$ & $0.0689\pm0.0009$ & $88.83\pm0.53$\\
4 layers, $25.6$M & 2 & $0.0563\pm0.0007$ & $0.0185\pm0.0005$ & $37.19\pm0.28$\\
8 layers, $25.6$M & 1 & $0.0402\pm0.0005$ & $0.0024\pm0.0002$ & $9.57\pm0.10$\\
8 layers, $25.6$M & 2 & $0.0397\pm0.0005$ & $0.0018\pm0.0001$ & $8.49\pm0.09$\\
\bottomrule
\end{tabular}
\caption{\textbf{Discovering a factorization: scaling toward informed Bayes.}
All quantities are in bits, averaged over the same $3840$ test
sequences of length $1025$, excluding the $256$ development sequences.
Next-token regret scores prediction of token $1025$; the gap subtracts the
informed Bayesian regret.
The cumulative gap sums the loss difference over predictions of tokens
$2$ through $1025$. Uncertainties are one standard error across
sequences, using paired differences for both gap columns.}
\label{tab:hidden-factor-scaling}
\end{table}

\begin{figure}[tbp]
\centering
\includegraphics[width=\linewidth]{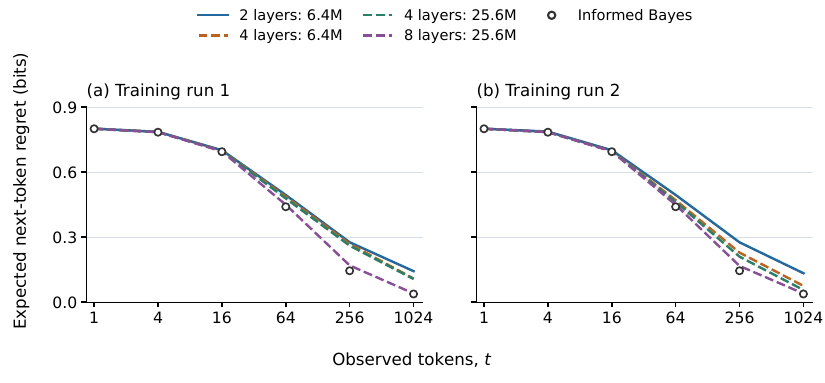}
\caption{\textbf{Discovering a factorization: size and training.}
Left: training run $1$; right: training run $2$. Each panel shows the
two-layer transformer as a solid curve and the larger transformers
as dashed curves; colors distinguish their sizes and training
lengths. Legend numbers give training sequences in millions. Unconnected open
circles show informed Bayes. All comparisons use
the same $3840$ test sequences and axes. Shading and bars give one standard
error across sequences.}
\label{fig:hidden-factor-scaling}
\end{figure}

\subsection{A known base law: uniform-base source}
\label{app:uniform-base-comparison}

With a known, uniform base law, this source tests whether the transformer
learns the prior on the rows. Increasing the model from two to four layers
at the same training budget improves agreement at all four displayed
depths (Figure~\ref{fig:known-structure-scaling}). At $t=1024$, the
four-layer gaps range from $0.00007$ to $0.02762$ bits in run~1 and
$0.00003$ to $0.01460$ bits in run~2. In particular, both models nearly
close the two-layer model's $0.030$-bit endpoint gap at $L=0$.
Their largest gaps over the evaluated contexts are $0.037\pm0.016$ and
$0.041\pm0.017$ bits, both at $L=8,t=64$.

\begin{figure}[tbp]
\centering
\includegraphics[width=\linewidth]{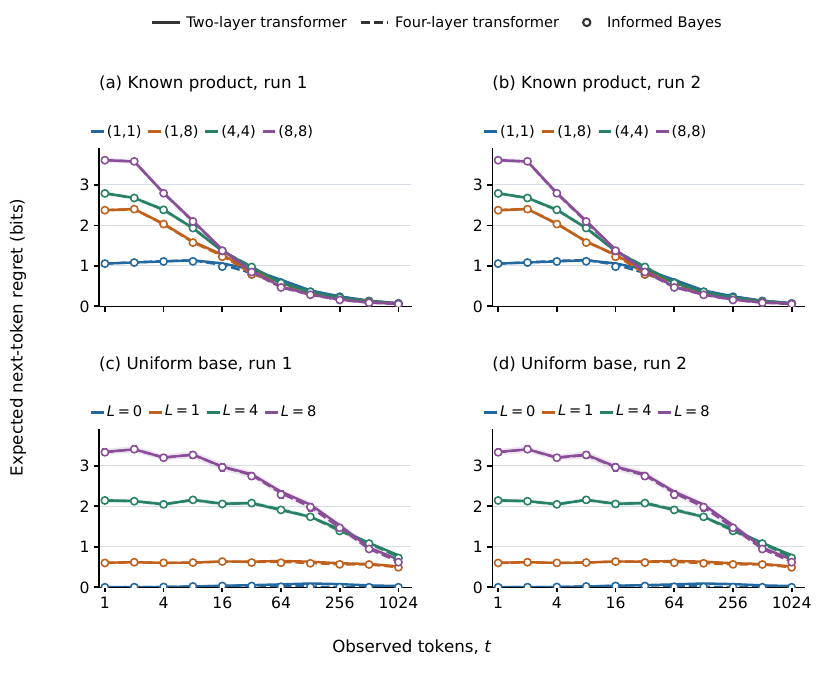}
\caption{\textbf{Known structure: increasing model size at fixed training budget.}
Rows show the known $8\times8$ product encoding and the uniform-base
source; columns show the two independent four-layer runs. Solid curves
show the same two-layer baseline in both columns, dashed curves the
four-layer model, and open circles informed Bayes. All models receive
$6.4$ million training sequences from the corresponding prior and use
the fixed test sets in Table~\ref{tab:conditional-cohorts}. Colors identify
the same source settings as Figure~\ref{fig:bayes-summary}; shading and
bars give one standard error across test sequences.}
\label{fig:known-structure-scaling}
\end{figure}

Figure~\ref{fig:bayes-evidence}(b) shows the two-layer model's full-predictive
divergence. On $128$ independent full-prior sequences, the transient gives a
cumulative gap of $31.68\pm1.57$ bits over $1024$ predictions.
Together, the two panels compare the two-layer models in the known-product and
uniform-base experiments.

\begin{figure}[tbp]
\centering
\includegraphics[width=\linewidth]{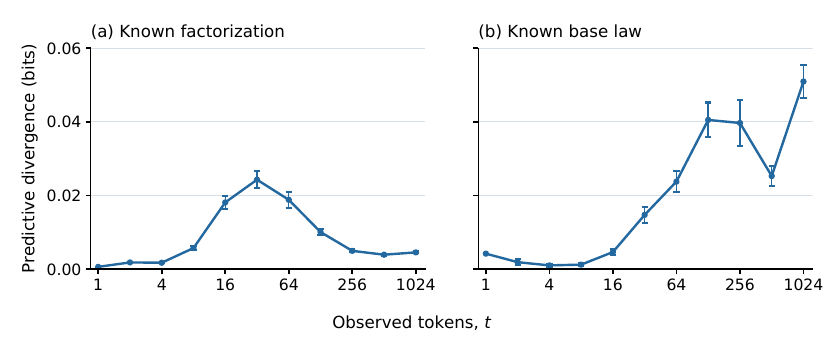}
\caption{\textbf{Informed-Bayes-to-transformer predictive divergence.}
Full-distribution comparison of the two-layer models for (a) the known $8\times8$
product factorization and (b) the uniform-base LSA source at $d=64$.
The horizontal axis counts observed tokens. Points show the mean of
$D_{\mathrm{KL},2}(Q_B\Vert Q_{\rm TF})$, in bits, over $128$ independent
sequences of length $1025$ from each training prior; bars give one standard
error. Lines join the evaluated contexts. Both panels use the same vertical
scale.}
\label{fig:bayes-evidence}
\end{figure}

\subsection{An unknown shared base law: shared-base source}
\label{app:shared-base-comparison}

More layers and more training brings both runs close to informed Bayes.
After $64$ tokens, the eight-layer models' gap ranges from
$-0.0017$ to $0.0066$ bits across all four depths and both runs. The
two-layer models have gaps as large as $0.220$ bits, while the four-layer
models trained on the same $6.4$ million sequences reduce the largest gap
to $0.040$ bits. Table~\ref{tab:shared-base-scaling} gives the paired
comparison at every depth.

Training the eight-layer models for $25.6$ rather than $12.8$ million
sequences reduces their regret after $1024$ tokens at every depth and
in both runs, by $0.0045$--$0.0096$ bits. At $(L,t)=(0,1024)$, the final
eight-layer gaps are $0.00409\pm0.00017$ and $0.00246\pm0.00014$ bits,
compared with $0.00899\pm0.00033$ and $0.00846\pm0.00030$ for the four-layer
models at $6.4$ million sequences. Figure~\ref{fig:shared-base-conditional}
shows the complete scaling sequence. All models are evaluated on the same
$256$ sequences per depth, with checkpoints fixed before evaluation.

\begin{table}[tbp]
\centering\small
\setlength{\tabcolsep}{3pt}
\begin{tabular}{ccrrrr}
\toprule
$L$ & Run & 2 layers, $6.4$M & 4 layers, $6.4$M & 4 layers, $25.6$M & 8 layers, $25.6$M\\
\midrule
0 & 1 & $0.2167\pm0.0088$ & $0.0395\pm0.0038$ & $0.0281\pm0.0030$ & $0.0065\pm0.0010$\\
0 & 2 & $0.2199\pm0.0090$ & $0.0287\pm0.0032$ & $0.0263\pm0.0029$ & $0.0034\pm0.0008$\\
1 & 1 & $0.1149\pm0.0099$ & $0.0169\pm0.0043$ & $0.0144\pm0.0039$ & $0.0066\pm0.0015$\\
1 & 2 & $0.1213\pm0.0093$ & $0.0103\pm0.0039$ & $0.0079\pm0.0035$ & $0.0005\pm0.0013$\\
4 & 1 & $0.0122\pm0.0116$ & $0.0146\pm0.0057$ & $0.0103\pm0.0052$ & $0.0026\pm0.0026$\\
4 & 2 & $0.0227\pm0.0118$ & $0.0141\pm0.0055$ & $0.0129\pm0.0055$ & $0.0037\pm0.0021$\\
8 & 1 & $0.0892\pm0.0128$ & $0.0049\pm0.0048$ & $0.0056\pm0.0043$ & $-0.0017\pm0.0020$\\
8 & 2 & $0.0822\pm0.0130$ & $0.0103\pm0.0043$ & $0.0129\pm0.0049$ & $0.0028\pm0.0021$\\
\bottomrule
\end{tabular}
\caption{\textbf{Inferring shared structure: the gap to informed Bayes after
$64$ tokens.}
Next-token loss differences in bits at $\alpha=1$, averaged over $256$
paired sequences per depth. Uncertainties are one standard error across
sequences.}
\label{tab:shared-base-scaling}
\end{table}

\begin{figure}[tbp]
\centering
\includegraphics[width=\linewidth]{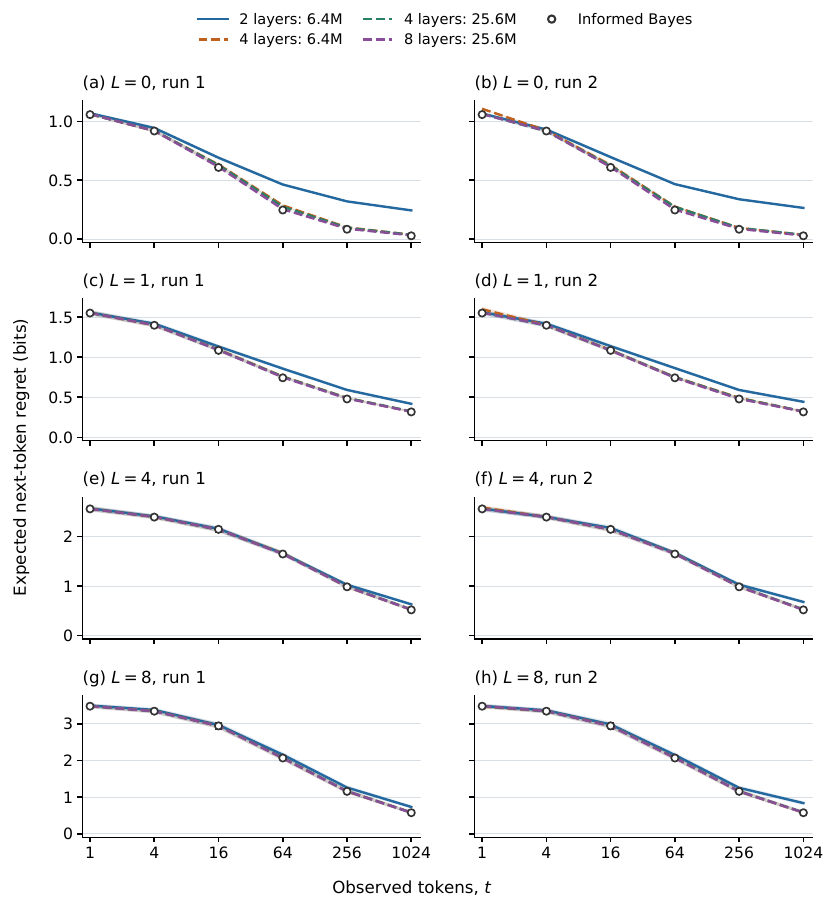}
\caption{\textbf{Inferring shared structure: size and training.}
Rows show source depths $L=0,1,4,8$; columns show the two training runs.
Each panel shows the two-layer transformer as a solid curve and the
larger transformers as dashed curves; colors distinguish their sizes
and training lengths. Legend numbers give training sequences in millions.
Unconnected open circles show informed Bayes at all six displayed contexts.
Each depth uses the same $256$ sequences for every predictor, and the
two runs share a vertical scale. Shading and bars give one standard error.}
\label{fig:shared-base-conditional}
\end{figure}

The long-sequence calculations in Appendix~\ref{app:shared-base-long-context}
apply the same independent-chain and integration-order checks at $256$ and
$1024$ tokens. All eight depth-and-context combinations use complete
test sets of $256$ sequences.
Their numerical error bounds are at most $0.000605$ bits and are reported
separately from source variation in Table~\ref{tab:shared-base-long-context}.

\subsubsection{Agreement at large $t$}
\label{app:shared-base-long-context}

Table~\ref{tab:shared-base-long-context} gives the comparisons at
all four depths after $256$ and $1024$ tokens,
on $256$ sequences each. Every sequence is evaluated twice with different
random seeds, using $16$ and $32$ integration nodes, or $32$ and $64$
nodes for $(L,t)=(0,1024)$. We report a mean only if every sequence and the mean
itself pass the numerical checks of Appendix~\ref{app:reference-numerics}.
Sequences from an earlier trial with eight sequences are included in the
$256$ and are not counted twice.

The numerical Monte Carlo uncertainty is distinct from variation between
sources. The table reports one source standard error for regret and paired
differences, and the separate numerical error bound. Informed Bayesian
optimality is prior-averaged, so fixed-setting differences can have either
sign.

\begin{table}[tbp]
\centering\small
\setlength{\tabcolsep}{3pt}
\begin{tabular}{crrrrr}
\toprule
$L$ & $t$ & Informed Bayesian regret & Run 1 gap & Run 2 gap & Numerical error bound\\
\midrule
0 & $256$ & $0.08323\pm0.00123$ & $0.00470\pm0.00051$ & $0.00303\pm0.00034$ & $0.000108$\\
0 & $1024$ & $0.02739\pm0.00040$ & $0.00409\pm0.00017$ & $0.00246\pm0.00014$ & $0.000033$\\
1 & $256$ & $0.47945\pm0.01265$ & $0.00508\pm0.00137$ & $0.00130\pm0.00114$ & $0.000443$\\
1 & $1024$ & $0.31776\pm0.01116$ & $0.00240\pm0.00089$ & $0.00315\pm0.00070$ & $0.000147$\\
4 & $256$ & $0.98501\pm0.03914$ & $0.00318\pm0.00151$ & $0.00263\pm0.00133$ & $0.000551$\\
4 & $1024$ & $0.52011\pm0.02917$ & $0.00327\pm0.00097$ & $0.00141\pm0.00104$ & $0.000242$\\
8 & $256$ & $1.16128\pm0.05668$ & $0.00078\pm0.00179$ & $-0.00103\pm0.00163$ & $0.000605$\\
8 & $1024$ & $0.57885\pm0.03987$ & $0.00158\pm0.00129$ & $0.00423\pm0.00116$ & $0.000359$\\
\bottomrule
\end{tabular}
\caption{\textbf{Inferring shared structure at large $t$.}
All quantities are in bits and use $256$ paired sequences per row.
The transformer gaps use the eight-layer $400000$-update checkpoints.
The last column is the numerical error bound of the Bayesian value, not a
source standard error; its uncertainty is shared by the two model
comparisons in that row.}
\label{tab:shared-base-long-context}
\end{table}

At
$L=0,t=1024$, four-layer models after $100000$ updates have gaps
$0.00899\pm0.00033$ and $0.00846\pm0.00030$ bits. Extending training
to $400000$ gives $0.00886\pm0.00032$ and $0.00843\pm0.00029$.
The final eight-layer models reduce these gaps to $0.00409\pm0.00017$
and $0.00246\pm0.00014$ bits. This comparison uses both runs and the models
at the end of training. No model was selected by its test performance.

\subsection{Inferring the source family: pooled training}
\label{app:pooled-comparison}

The pooled transformers closely reproduce the informed Bayesian mixture.
After four tokens, increasing from two to four layers reduces the gap from
$0.0197$ to $0.0052$ bits in run~1 and from $0.0190$ to $0.0025$ bits
in run~2. After one token, the four-layer gaps are $0.0080$ and $0.0067$
bits. Thus the same agreement holds when prediction
requires both inference within each source family and inference over the
family itself.

Each model is trained without family labels on equal proportions of product,
shared-base and uniform-base sources, receiving $6.4$ million sequences per
family. For the family probability $Q_f(x^t)$, integrated over its prior, informed Bayes uses
\begin{align}
 w_f(t)&=\frac{Q_f(x^t)}{\sum_{g=1}^3 Q_g(x^t)}
       =\frac{2^{-\mathcal L_f(x^t)}}{\sum_{g=1}^3 2^{-\mathcal L_g(x^t)}},
       \label{eq:pooled-family-weights}\\
 Q_{\rm mix}(b\mid x^t)&=\sum_{f=1}^3w_f(t)Q_f(b\mid x^t).
       \label{eq:pooled-prediction}
\end{align}
All three initial-token probabilities are $1/64$, and every component
probability and prediction concern the same observed history.

The test set has $1024$ independent sequences from each family, drawn from its full prior.
We average within families and then weight families equally; the variance of
this average is the sum of family-mean variances divided by nine. Both runs
use the same test set. Figure~\ref{fig:pooled-training-runs} compares both
model sizes with informed Bayes through $t=1024$.
Bayesian estimates at $t=16,64,256,1024$ use a fixed subset
of $64$ histories per family. At $t=1024$, the four-layer models have regrets
$0.3787$ and $0.3751$ bits, improving on the two-layer models by
$0.0360\pm0.0013$ and $0.0357\pm0.0013$ bits.
Table~\ref{tab:pooled-bayes-loss} gives the two-layer models'
cumulative comparison through four predictions. Conditional family
differences may have either sign because Bayesian optimality applies to the
complete mixture prior.

\begin{figure}[tbp]
\centering
\includegraphics[width=\linewidth]{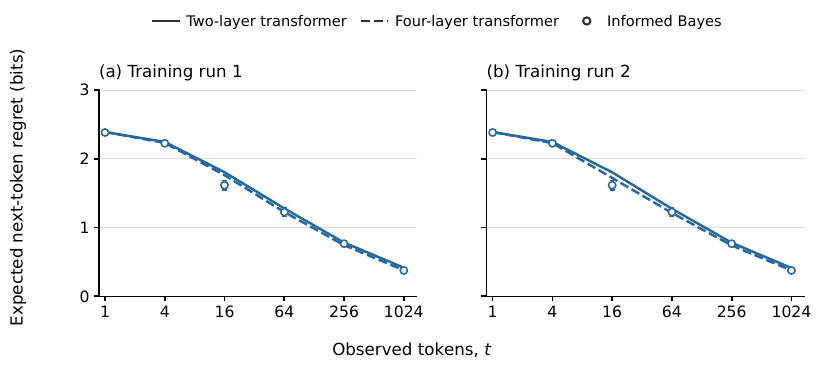}
\caption{\textbf{Inferring the source family: two independent runs.}
Each panel compares a two-layer transformer (solid curve) and a four-layer
transformer (dashed curve), trained on the same $19.2$ million indexed
sequences from the three $64$-symbol families. Open circles show the
informed Bayesian mixture through $t=1024$.
Transformer curves and Bayes at $t=1,4$ use $1024$ histories per family;
Bayes at longer contexts uses $64$ per family. Both runs share their test
histories. Shading and bars give one standard error across histories.}
\label{fig:pooled-training-runs}
\end{figure}

\begin{table}[tbp]
\centering
\caption{\textbf{Two-layer cumulative gap to the informed Bayesian
mixture.} Loss differences over four predictions, in bits. Each family has
$1024$ independent test
sequences; the final row averages the three family means equally.
Values are means $\pm$ one paired standard error.}
\label{tab:pooled-bayes-loss}
\begin{tabular}{lrr}
\toprule
Generating family & Training run $1$ & Training run $2$ \\
\midrule
Product & $0.023\pm0.011$ & $-0.018\pm0.011$ \\
Shared-base LSA & $0.105\pm0.014$ & $0.131\pm0.015$ \\
Uniform-base control & $-0.027\pm0.009$ & $-0.005\pm0.009$ \\
\midrule
Equal-family average & $0.034\pm0.007$ & $0.036\pm0.007$ \\
\bottomrule
\end{tabular}
\end{table}

\paragraph{Reference validation.}
\label{app:pooled-predictive-validation}
Product and uniform-base probabilities use the checked row integrals and
common-depth averaging. At $t=1,4$, shared-base probabilities use exhaustive short-pattern
integration over exponent, depth and permutation. Its raw normalization
error is below $1.5\cdot10^{-11}$, and an independent initial-prediction
check agrees within $5.5\cdot10^{-13}$. Sequential mixture predictions
and direct $Q_{\rm mix}(x^{t+1})/Q_{\rm mix}(x_1)$ agree within
$7.1\cdot10^{-12}$ bits on $24$ sequences through four predictions.
For longer histories, annealed sequential Monte Carlo estimates the
shared-base sequence and successor probabilities jointly. Their ratio
gives the predictive probabilities used in the mixture.
The estimate at $t=1024$ uses eight repetitions per history.
At $t=16$, each of the $192$ histories uses $32$ independent repetitions
with $8192$ particles. Separate particle, exponent and row-integration
refinements give a total numerical allowance of $0.00205$ bits for the
pooled mean, below the $0.01/3$-bit budget. This numerical allowance is
distinct from the standard error across test histories.

\subsubsection{The cost of family uncertainty}
\label{app:application-family}

Following Bayesian model-averaging comparisons
\citep{zhang2025bayesian,bai2023transformers,panwar2024bayesian,deora2025occam},
let $\mathcal L_F(n)$ score tokens $2,\ldots,n+1$ with the generating
family known but its parameters unknown. Define
\begin{equation}
 C_B(n)=\mathbb E[\mathcal L_{\rm mix}(n)-\mathcal L_F(n)],\qquad
 C_{\rm TF}(n)=\mathbb E[\mathcal L_{\rm TF}(n)-\mathcal L_F(n)].
 \label{eq:application-family-cost}
\end{equation}
Every family assigns $X_1$ the same probability, so
\[
 C_B(n)=I(F;X_2,\ldots,X_{n+1}\mid X_1)
       =\log_2 3-H(F\mid X_1,\ldots,X_{n+1}).
\]
Thus $\log_2 3$ bounds the cost; $C_{\rm TF}-C_B$ is the cumulative
transformer--mixture gap in Table~\ref{tab:pooled-bayes-loss}.

For $n\le4$, validated component likelihoods allow a population sum.
Separate coordinate relabeling represents each history by equality
patterns $u,v$, with multiplicity $M(u,v)=(8)_{k_u}(8)_{k_v}$;
there are at most $52^2=2704$ pairs for five tokens. For representatives
$x(u,v)$,
\[
 C_B(n)=\frac13\sum_f\sum_{u,v}M(u,v)Q_f(x(u,v))
       \log_2\frac{Q_f(x(u,v))}{Q_{\rm mix}(x(u,v))}.
\]
Total-mass checks pass without renormalization. For $n=1,2,3,4$ the
costs are $0.01163,0.04415,0.10227,0.17562$ bits, without test-sample
uncertainty. Adding the paired transformer--mixture sample mean gives
$C_{\rm TF}(4)=0.2093\pm0.0066$ and $0.2116\pm0.0069$ bits for
runs~1 and~2. This paired estimate has variance equal
to the sum of within-family paired variances divided by $9\cdot1024$;
numerical error is assessed separately. Informed Bayes explains most of the
early cost.

\paragraph{Residual family ambiguity.}
The limiting cost obeys
\begin{equation}
 \log_2 3-\frac{16}{27}
 \;\leq\;\lim_{n\to\infty}C_B(n)
 \;<\;\log_2 3.
 \label{eq:family-cost-limit}
\end{equation}
An infinite trajectory identifies its positive transition matrix $T$
through empirical transition frequencies. Conditional on $(T,X_1)$,
future tokens contain no further information about $F$, so the limit
is $I(F;T\mid X_1)$. The product prior lies on the lower-dimensional
set of product matrices; the depth-zero shared-base and uniform-base
priors give identical nonuniform Zipf rows and the uniform matrix,
respectively. These components are mutually singular, so $F$ is
identified almost surely there.

The positive-depth shared-base and uniform-base priors instead have
positive densities over all positive transition matrices. At depth one,
uniform-simplex rows multiplied by positive base weights and normalized
are an invertible smooth image of the simplex interior; conditioning
on the other factors extends absolute continuity to larger depths.
Thus both families remain possible on this continuous component,
whose prior mass is $2(1/3)(8/9)=16/27$. At most one bit of family
entropy remains there and strictly positive entropy does remain.
Hence $0<H(F\mid T,X_1)\le16/27$, proving the bound. Even unlimited
data need not identify the generating family.

Across all six source constructions, the detailed comparisons support the
same progression as Section~\ref{sec:bayes-evidence}. The two-layer transformer
closely tracks informed Bayes when the relevant structure is directly
accessible; more layers and more training closes the gap when the encoding
or shared structure must be inferred; and pooled training extends the same
correspondence to inference over source families.

\section{Details for Section~\ref{sec:context-application}: four further tests}
\label{app:application-protocols}

This appendix follows the four scenarios in
Section~\ref{sec:context-application} and Figure~\ref{fig:applications}:
source mismatch, a time-varying source, an unseen context, and the
emergence of copying. Each subsection first specifies the scenario and the
informed Bayesian response, then gives the complete transformer comparison
and its numerical checks.

\subsection{Common experimental and numerical protocol}

The transformer's weights are fixed before evaluation. Informed Bayes
continues to use the family and prior represented in the transformer's
training sequences. During evaluation, neither predictor receives the test
parameters or a signal identifying the test condition. Test sets are
independent at the level of complete source draws or sequences. Paired
conditions share histories, source parameters and the random numbers that drive the source wherever
their test laws permit.

All losses and regrets are in bits. Means compare predictions on the same
histories, and uncertainty is computed across whole sequences, preserving all
tested times and paired conditions. Unless a caption states otherwise, curves
show pointwise $95\%$ intervals and textual $\pm$ values give one standard
error. Training-run variation is reported separately from sequence-to-sequence
variation.

Exact finite sums are used when available. Calculations with continuous or
latent variables are repeated with independent integration rules or chains,
and, except for explicitly marked preliminary estimates, test-set means are
reported only after every required reference passes the criteria in Appendix~\ref{app:bayes-evidence-details}. Saved forecasts are
checked in raw probability space; none is clipped or renormalized for
agreement.

\subsection{Source mismatch}

The fixed predictors receive sequences from a product test source whose $64$
states have random token labels, rather than the coordinate-revealing labels
used in training. This source has the same entropy and predictability as the
ordinary product source, but its labels hide the coordinate factorization.
Informed Bayes retains the training product prior and approaches the product
projection derived below. At
$(L_1,L_2)=(1,1)$ and $(4,4)$, its predicted stationary regrets are
$0.966\pm0.003$ and $2.237\pm0.007$ bits. At depth four and $t=1024$,
informed Bayes has regret $2.299$ bits. The two-layer transformer has
$2.309$ bits, and the two four-layer runs have $2.300$ and $2.305$ bits.

\subsubsection{Matched product test conditions}
\label{app:product-shifts}

We use run~1 of the known-encoding $8\times8$ product transformer (two
layers, width $128$, $544256$ parameters, $100000$ updates) and its informed
Bayesian predictor. Bayes assigns independent uniform depth priors on
$\{1,\ldots,8\}$ to the two coordinates, with each coordinate's depth shared
across its rows.

For the random-label condition we also evaluate both four-layer
known-encoding product models: width $256$, eight heads and $3454464$
parameters. All three transformers have $100000$ updates and see $6.4$
million sequences from the same training prior, with independently drawn
training sequences. The two four-layer runs use the same $1024$ test
histories per depth and the same Bayesian forecasts as the two-layer model.
We report the runs separately.

At each of $(L_1,L_2)=(1,1)$ and $(4,4)$, we generate $1024$ matched sets of
four test histories: a stationary product source, a product source observed
through random joint labels, a time-varying source with one new coordinate
matrix after the change point, and a time-varying source with two. The initial
state is uniform. The random-label condition uses a fresh uniform permutation
$\eta$ of the $64$ labels,
\[
 T^\eta(b\mid a)=T\bigl(\eta^{-1}(b)\mid\eta^{-1}(a)\bigr),
 \qquad T=T_1\otimes T_2,
\]
without shuffling token positions. Common initial states and uniform
random numbers pair all four conditions. Histories extend untruncated through the
transformer's limit of $1024$ tokens.

Regret is the KL divergence from the actual current transition row to the
forecast. We also measure the paired transformer-minus-Bayes regret and
$D_{\mathrm{KL},2}(Q_B\Vert Q_{\mathrm{TF}})$.

Pairing, the random label mapping and the parameter boundary are checked
against the full $64$-state matrices. Sixteen test histories spanning all
conditions and depths pass wider-integration and exact depth-one checks; the
largest depth-one probability discrepancy is below $10^{-13}$. All saved
forecasts pass raw normalization checks and agree on the shared prefixes.
The four-layer forecasts also pass separate-prefix causal checks. Their
regrets, predictive divergences and paired standard errors are recomputed
from the saved forecasts on all test histories.
\subsubsection{Product projection and residual information loss}
\label{app:mismatch-limits}

The product projection identifies the exact information excluded by the
visible factorization and therefore predicts the random-label mismatch cost.

Let $P_{uv,u'v'}>0$ be the true transition matrix on the visible
$8\times8$ states, with stationary law $\mu_{uv}$ and stationary
pair law $J_{uv,u'v'}=\mu_{uv}P_{uv,u'v'}$. The assumed model is
$Q_{A,B}(u',v'\mid u,v)=A_{uu'}B_{vv'}$. Maximizing its stationary
expected log likelihood separately over each stochastic row gives
\begin{align*}
 A^*_{uu'}&=\frac{\sum_{v,v'}J_{uv,u'v'}}{\sum_v\mu_{uv}}
             =P(U'=u'\mid U=u),\\
 B^*_{vv'}&=\frac{\sum_{u,u'}J_{uv,u'v'}}{\sum_u\mu_{uv}}
             =P(V'=v'\mid V=v).
\end{align*}
The coordinate processes need not themselves be Markov: these are
conditionals of stationary pairs. Positivity makes the full chain
ergodic and both row-wise maximizers unique.

For every history, the assumed likelihood factors into coordinate
transition-count likelihoods. Independent factor priors therefore give
independent posteriors even when the \emph{true} coordinates are
dependent. Empirical pair frequencies converge to the marginals of
$J$. Each factor prior gives positive mass to every neighborhood of
every positive transition matrix, since its depth-one component has
independent uniform Dirichlet rows. The limiting log likelihood has
the unique maximizer $(A^*,B^*)$, so posterior mass outside its
neighborhoods vanishes and the predictive matrix approaches
$Q^*=A^*\otimes B^*$. Averaging over depth shared across rows changes
the predictions at finite $t$, but not this limit.

\paragraph{Residual information loss.}
The stationary regret of this product projection is
\begin{align*}
 R_\infty
 &=\sum_{u,v,u',v'}J_{uv,u'v'}
           \log_2\frac{P_{uv,u'v'}}{A^*_{uu'}B^*_{vv'}}\\
 &=H(U'\mid U)+H(V'\mid V)-H(U',V'\mid U,V)\\
 &=I(U';V\mid U)+I(V';U\mid V)+I(U';V'\mid U,V).
\end{align*}
These are the predictive dependencies excluded by the assumed
factorization. They vanish with visible product labels. The random-label test
law has the same entropy rate $h=H(U',V'\mid U,V)$ but generally contains
predictive dependencies that visible-product Bayes excludes, making
$R_\infty>0$. Since uniform prediction is itself a product of two coordinate predictions,
$0\le R_\infty\le6-h$. This bounds the optimal stationary product
regret, not the regret at every finite $t$. The stationary average also
differs from the realized current-row regret, which varies with state.

For each saved random-label test condition, we calculate the stationary law
from its two coordinate matrices and saved label mapping. The finite sums
above then give the following means $\pm$ one standard error over the $1024$
test sources per depth, without fitting any prediction curve:
\[
 \begin{array}{c|ccc}
 (L_1,L_2)&h&\text{limiting expected loss}&R_\infty\\\hline
 (1,1)&4.9549\pm0.0039&5.9205\pm0.0010&0.9657\pm0.0035\\
 (4,4)&3.1080\pm0.0126&5.3449\pm0.0131&2.2369\pm0.0069
 \end{array}
\]
A direct $64$-state KL calculation agrees with the entropy formula
within $2.1\cdot10^{-15}$ bits; an independent stationary solve agrees
within $1.1\cdot10^{-15}$. These are unbounded-history informed-Bayesian
limits. The experiments start from a uniform state and stop at $1024$
tokens; their curves at finite $t$ approach the limits from that common
initial law.
\subsubsection{Complete product comparisons}

Table~\ref{tab:product-shift-endpoints} gives all eight two-layer endpoint comparisons,
and Figure~\ref{fig:product-shifts-depth1} gives the complete depth-one curves.
For this model, final predictive divergence is below $0.005$ bits in every
condition. On the random-label test sequences, its gap reaches
$0.320$ bits at depth one and $0.209$ bits at depth four, both after $32$
tokens, before the gap becomes small.

Both four-layer runs reduce this intermediate-context gap
(Table~\ref{tab:product-mismatch-larger}). Their largest gaps on the tested
grid occur after $16$ tokens: $0.129$ and $0.137$ bits at depth one, and
$0.087$ and $0.095$ bits at depth four. Predictive divergence
$D_{\mathrm{KL},2}(Q_B\Vert Q_{\mathrm{TF}})$ also decreases at $t=32$:
from $0.0795$ to $0.0117$ and $0.0115$ bits at depth one, and from
$0.0724$ to $0.0137$ and $0.0140$ bits at depth four.

\begin{table}[tbp]
\centering\footnotesize
\setlength{\tabcolsep}{5pt}
\begin{tabular}{ccrrr}
\toprule
$(L_1,L_2)$ & $t$ & Two layers & Four layers, run 1 & Four layers, run 2 \\
\midrule
$1,1$ & $32$ & $0.3201\pm0.0053$ & $0.0899\pm0.0021$ & $0.0870\pm0.0022$ \\
$1,1$ & $1024$ & $0.0004\pm0.0005$ & $-0.0051\pm0.0004$ & $-0.0004\pm0.0007$ \\
$4,4$ & $32$ & $0.2090\pm0.0084$ & $0.0617\pm0.0035$ & $0.0624\pm0.0036$ \\
$4,4$ & $1024$ & $0.0100\pm0.0016$ & $0.0003\pm0.0008$ & $0.0061\pm0.0015$ \\
\bottomrule
\end{tabular}
\caption{\textbf{Product mismatch with two- and four-layer transformers.}
Paired transformer-minus-Bayes regret, mean $\pm$ one standard error in
bits, on the same $1024$ random-label histories per depth. All models have
the same training prior and sequence budget. Negative differences are
possible because the test source lies outside the training family.}
\label{tab:product-mismatch-larger}
\end{table}

\begin{table}[tbp]
\centering\footnotesize
\setlength{\tabcolsep}{3.5pt}
\begin{tabular}{llrrr}
\toprule
$(L_1,L_2)$ & Condition & Transformer & Bayes & Paired difference \\
\midrule
$1,1$ & Visible product labels & $0.0759\pm 0.0011$ & $0.0723\pm 0.0010$ & $0.0036\pm 0.0003$ \\
$1,1$ & Random joint labels & $1.0396\pm 0.0095$ & $1.0392\pm 0.0095$ & $0.0004\pm 0.0005$ \\
$1,1$ & One new matrix after $t=256$ & $0.1561\pm 0.0026$ & $0.1530\pm 0.0026$ & $0.0031\pm 0.0003$ \\
$1,1$ & Two new matrices after $t=256$ & $0.2409\pm 0.0037$ & $0.2378\pm 0.0037$ & $0.0031\pm 0.0004$ \\
$4,4$ & Visible product labels & $0.0664\pm 0.0016$ & $0.0619\pm 0.0015$ & $0.0044\pm 0.0003$ \\
$4,4$ & Random joint labels & $2.3093\pm 0.0222$ & $2.2993\pm 0.0217$ & $0.0100\pm 0.0016$ \\
$4,4$ & One new matrix after $t=256$ & $0.2389\pm 0.0074$ & $0.2361\pm 0.0075$ & $0.0027\pm 0.0005$ \\
$4,4$ & Two new matrices after $t=256$ & $0.4164\pm 0.0106$ & $0.4171\pm 0.0107$ & $-0.0007\pm 0.0007$ \\
\bottomrule
\end{tabular}
\caption{\textbf{Regret after $1024$ tokens under four product test sources.}
The transformer has two layers. Mean $\pm$ one standard error in bits over
$1024$ paired test sets per depth.
The last column is the paired transformer--Bayes difference; informed Bayes
retains its training product prior in every condition.}
\label{tab:product-shift-endpoints}
\end{table}

\begin{figure}[tbp]
\centering
\includegraphics[width=\linewidth]{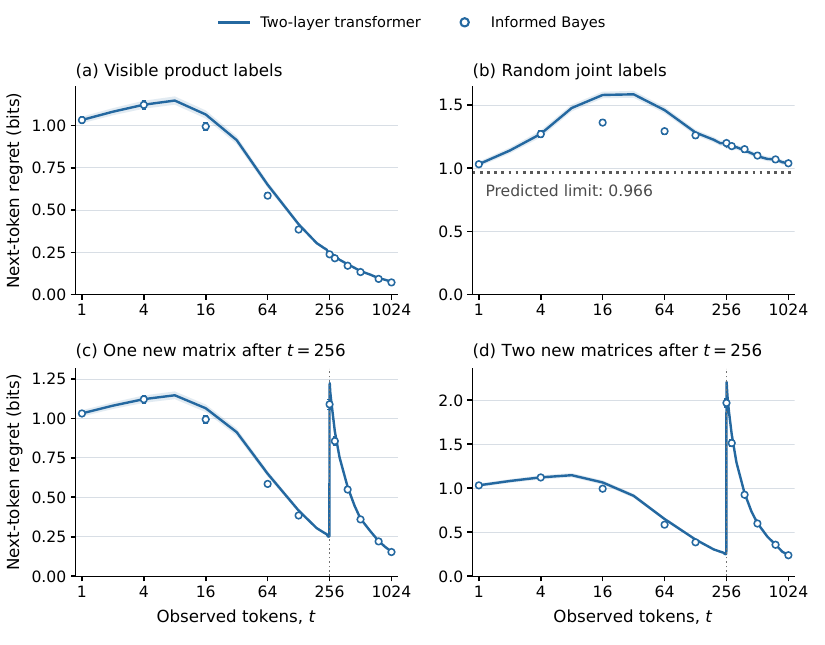}
\caption{\textbf{Complete product comparisons at depth one.}
Source depths are $L_1=L_2=1$. Solid curves show the two-layer transformer and
unconnected open circles show informed Bayes; all conditions average $1024$
paired draws, with pointwise $95\%$ intervals. The dotted random-label limit,
$0.966$ bits, is calculated from the test-source matrices.}
\label{fig:product-shifts-depth1}
\end{figure}

\subsection{Time-varying source}

Both fixed predictors receive sequences from a time-varying product source.
One condition uses a new transition matrix for one coordinate after $256$
tokens; the other uses new matrices for both coordinates. The state,
encoding and observed history continue across the boundary, so the first new
matrix generates $X_{257}$ from the observed $X_{256}$. Neither predictor is
reset or receives a change signal. The new matrices have the same fixed
depths, and the one- and two-matrix conditions use the same new first matrix.
Informed Bayes updates under its stationary product prior and predicts that
two independent changes produce twice the additional regret of one. The
transformer follows both the Bayesian adaptation curve and this factor-of-two
prediction.

\subsubsection{Additivity for a time-varying source}
\label{app:product-shift-prediction}

The factor of two in~\eqref{eq:product-shift-additivity} is an exact
population identity at every $t$ for additional regret in the paired
experiment of Appendix~\ref{app:product-shifts}.

\paragraph{Factorization and additivity.}
The independent coordinate priors imply, for every history,
\[
 Q_B(u',v'\mid u^t,v^t)
 =Q_B^{(1)}(u'\mid u^t)Q_B^{(2)}(v'\mid v^t).
\]
A common depth couples rows within each coordinate, but not the two
coordinates. Bayes continues to  use this stationary product model when the
test source varies with time. The actual post-boundary next-token law also
factors into current rows $p_1,p_2$. Therefore, for every sequence,
\[
 D_{\mathrm{KL},2}(p_1\otimes p_2\Vert q_1\otimes q_2)
 =D_{\mathrm{KL},2}(p_1\Vert q_1)+D_{\mathrm{KL},2}(p_2\Vert q_2).
\]
Let $r_u(t)$ be the expected regret of a coordinate that continues with its
first matrix and $r_c(t)$ that of a coordinate using a new matrix. The
coordinate alphabets, depth laws,
initial-state laws and predictor priors are identical, so
\[
 r_B^{\mathrm{stationary}}=2r_u,\qquad
 r_B^{\mathrm{one}}=r_c+r_u,\qquad
 r_B^{\mathrm{both}}=2r_c.
\]
Subtracting the stationary regret proves the doubling identity.
No equilibrium start or large-context approximation is required.
Different coordinate laws would still yield additive contributions,
but not generally equal contributions. A transformer's predictions
need not factor, so the same relation is a separate behavioral test.

\paragraph{Paired contrast and simultaneous uncertainty.}
For source draw $i$ and predictor $M$, define
\[
 Z_{M,i}(t)=\rho_{M,i}^{\mathrm{both}}(t)
       -2\rho_{M,i}^{\mathrm{one}}(t)
       +\rho_{M,i}^{\mathrm{stationary}}(t),
\]
where $\rho$ is conditional next-token regret. The paired construction makes
$Z_B$ the difference between the two coordinates' additional regrets. Hence
$\mathbb E Z_B=0$, although individual $Z_{B,i}$ need not vanish. We estimate
uncertainty from whole paired test sets, preserving correlations across
conditions and times.

A centered nonparametric bootstrap resamples the test sets of the two
depths independently, retaining both predictors and all $14$ post-change
times ($t=256,\ldots,1024$ on the tested grid) within each draw.
For each of $10000$ resamples we take the largest absolute mean
deviation divided by its reported paired standard error over all
$56$ depth/predictor/time combinations. The $95\%$ quantile is
$3.156$; the largest observed standardized discrepancy is $2.272$.
All contrasts lie within these simultaneous bands.
At depth four and $t=1024$, transformer additional regret with two new
matrices is $0.3500$ bits versus the predicted $0.3451$ from twice the
one-new-matrix condition; the paired contrast is $0.0049\pm0.0106$ bits (one
standard error).
Thus the transformer reproduces the additive mean relation derived from the
source factorization.

\begin{figure}[tbp]
\centering
\includegraphics[width=\linewidth]{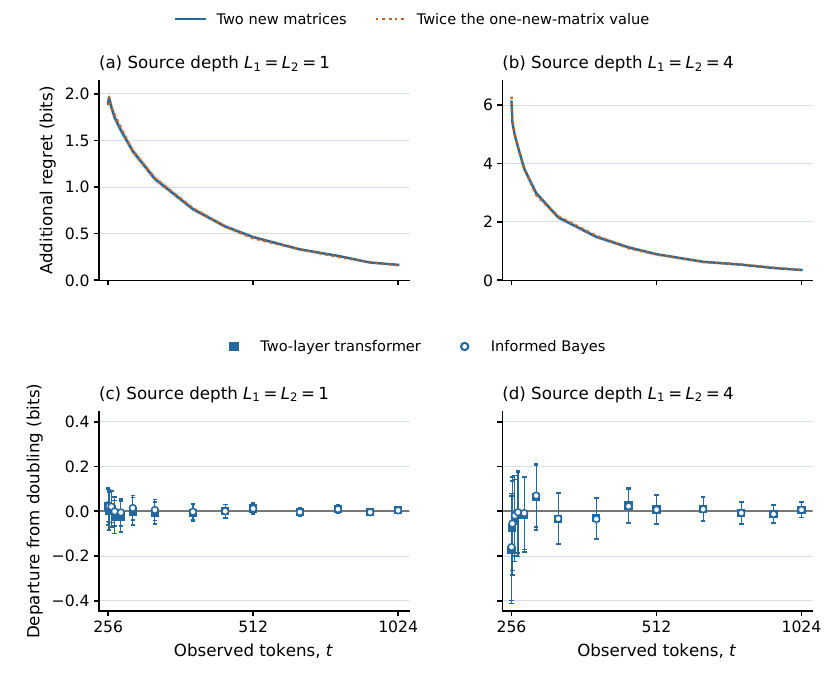}
\caption{\textbf{Doubling prediction and paired departures.}
Top: transformer additional regret with two new matrices (solid) and twice
the one-new-matrix value (dotted), at both depths. Bottom: $Z_M(t)$ for the
transformer (filled squares) and informed Bayes (open circles), without
connecting lines. Bars give approximate $95\%$ simultaneous intervals over
both depths, predictors and all post-change grid points. Every interval includes zero. Each depth averages $1024$
paired draws; the multiplier two is fixed by the Bayesian identity.}
\label{fig:product-shift-additivity}
\label{fig:product-shift-additivity-contrast}
\end{figure}
\subsubsection{Exact coin control}
\label{app:coin-shifts}

\paragraph{Models and scoring.}
We use the coin transformers trained with equal mass on
$\{0.2,0.8\}$ or a uniform bias on $[0,1]$ (two layers, width $64$,
$116736$ parameters, $20000$ updates). Three independent test sets contain
$4096$ sequences of length $257$ each: iid Bernoulli$(0.4)$, or a
time-varying source with bias $0.2$ for the first $\tau=64$ or $128$
tokens and bias $0.8$ thereafter. Both trained models receive the same
sequences, keep all tokens and receive no signal at the boundary. The
new bias first
affects $X_{\tau+1}$, so the prediction at $t=\tau$ uses only old
tokens from before the change to predict a token from after it. At every $t=1,\ldots,256$,
we score conditional expected log loss relative to the actual
next-token bias. Intervals are across independent test sequences;
there is one trained model per prior.

\paragraph{Exact count formulas.}
For equal prior weights on $a$ and $1-a$, $0<a<1/2$, a prefix with
$n_1$ ones in $n$ tokens has posterior odds
$[(1-a)/a]^{2n_1-n}$ in favor of $1-a$. Thus
\begin{equation}
 q_n(n_1)=a+\frac{1-2a}{1+[(1-a)/a]^{n-2n_1}};
 \qquad q_n^{\mathrm{uniform}}(n_1)=\frac{n_1+1}{n+2}.
 \label{eq:mismatch-coin-predictive}
\end{equation}
Write $b_{n,\theta}(n_1)=\binom n{n_1}\theta^{n_1}(1-\theta)^{n-n_1}$ and let
$D_{\mathrm{bin}}$ denote the binary KL divergence in bits. Under iid test
bias $\theta$, the exact
expected next-token regret is
\begin{equation}
 r_B(n;\theta)=\sum_{n_1=0}^{n}b_{n,\theta}(n_1)\,
              D_{\mathrm{bin}}\bigl(\theta\Vert q_n(n_1)\bigr).
 \label{eq:mismatch-coin-risk}
\end{equation}
Let $N_1(n)=\sum_{i=1}^nX_i$ denote the random number of ones.
For $\theta<1/2$, the law of large numbers gives $q_n(N_1(n))\to a$;
bounded convergence, since $q_n\in[a,1-a]$, gives
$r_B(n;\theta)\to D_{\mathrm{bin}}(\theta\Vert a)$. At $a=0.2,\theta=0.4$ the residual is
$0.151$ bits. The uniform prior instead includes the true bias and
its regret vanishes.

After $\tau$ tokens at bias $\theta_-$ and $m$ at $\theta_+$,
the two independent numbers of ones have laws $b_{\tau,\theta_-}$ and
$b_{m,\theta_+}$. Their convolution gives the exact regret for this time-varying test
source:
\begin{equation}
 r_{B,\tau}^{\mathrm{change}}(m)
 =\sum_{n_1=0}^{\tau+m}
       (b_{\tau,\theta_-}*b_{m,\theta_+})(n_1)\,
       D_{\mathrm{bin}}\bigl(\theta_+\Vert q_{\tau+m}(n_1)\bigr).
 \label{eq:coin-shift-risk}
\end{equation}
Before the switch, use~\eqref{eq:mismatch-coin-risk} with $\theta_-$.
These are population expectations with no fitted parameters or
Monte Carlo error. Enumerating all binary histories through eight
tokens, including a switch after three, verifies the count
formulas within $1.3\cdot10^{-15}$ bits.

\paragraph{The learned prior determines adaptation.}
The uniform-prior transformer follows the exact Bayesian regrets closely: the
maximum absolute paired mean regret difference across the three tested curves
is below $0.009$ bits. The two-point prior predicts persistent mismatch and
delayed adaptation, and the transformer reproduces both effects
(Figure~\ref{fig:coin-shifts}).
For iid bias $0.4$, its largest paired difference is $-0.0391$ bits
at $t=97$; for the $\tau=128$ change it reaches $-0.1741$ bits at
$t=228$. For $\tau=64$ the largest magnitude is $0.1425$ bits at
$t=150$. These are the largest values on the grid of tested $t$;
Table~\ref{tab:coin-shift-endpoints} reports all endpoints. On these two-bias
test sequences, the two-point transformer's smaller transient regret reflects
prediction outside the training prior.

\begin{figure}[tbp]
\centering
\includegraphics[width=\linewidth]{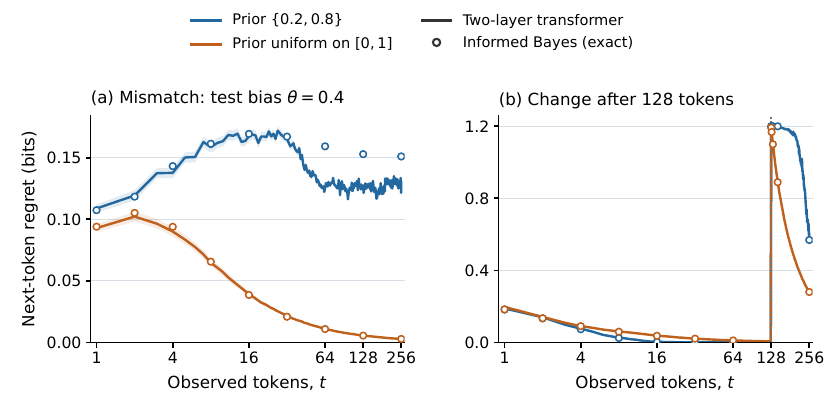}
\caption{\textbf{Exact coin predictions under mismatch and time variation.}
Regret under fixed bias $0.4$ (a) and under a time-varying source with bias
$0.2$ through $128$ tokens and $0.8$ thereafter (b). Solid curves are trained
transformers; open circles are exact informed-Bayesian expectations. Color
denotes the two-point (blue) or uniform (orange) training prior. Shading
gives pointwise $95\%$ intervals over $4096$ sequences; exact expectations
have no Monte Carlo error.}
\label{fig:coin-shifts}
\end{figure}

\begin{table}[tbp]
\centering\footnotesize
\setlength{\tabcolsep}{4pt}
\begin{tabular}{llrrr}
\toprule
Training prior & Test source & Transformer & Exact Bayes & Paired difference\\
\midrule
Two-point & $\theta=0.4$ & $0.13022\pm0.00040$ & $0.15109$ & $-0.02071\pm0.00040$\\
Uniform & $\theta=0.4$ & $0.00279\pm0.00006$ & $0.00279$ & $(9.51\pm2.32)\cdot10^{-6}$\\
Two-point & $0.2$ through $64$, then $0.8$ & $0.00059\pm0.00002$ & $3.75\cdot10^{-9}$ & $0.00059\pm0.00002$\\
Uniform & $0.2$ through $64$, then $0.8$ & $0.07964\pm0.00037$ & $0.08086$ & $-0.00170\pm0.00002$\\
Two-point & $0.2$ through $128$, then $0.8$ & $0.59820\pm0.00415$ & $0.56871$ & $0.03068\pm0.00581$\\
Uniform & $0.2$ through $128$, then $0.8$ & $0.27629\pm0.00066$ & $0.27985$ & $-0.00289\pm0.00003$\\
\bottomrule
\end{tabular}
\caption{\textbf{Coin regret after $256$ tokens, in bits.}
Transformer means and paired differences show one standard error
across $4096$ sequences. The Bayesian column is a population
expectation; the last column instead subtracts Bayesian regret on
each sampled prefix and therefore need not equal the difference
of the preceding columns.}
\label{tab:coin-shift-endpoints}
\end{table}

The saved float32 weights are evaluated with float64 log-softmax
and scoring. Checks reproduce sources from fixed seeds, verify raw
normalization and causal prefix predictions, and independently
recompute saved summaries. Bayesian sample means additionally pass
simultaneous Bernstein checks against the exact expectations over
all six curves and $256$ times. Predictions are neither clipped nor
renormalized to pass validation.

\subsection{An unseen context}
\label{app:application-transfer}

When the current token appears for the first time, no continuation from its
transition row has been observed. Shared-base Bayes nevertheless uses the
other rows to infer their common latent base; uniform-base Bayes instead
predicts uniformly. The transformers reproduce both responses. Across all
twelve shared-base depth--time conditions, both larger-model runs remain
within $0.015$ bits of informed Bayes. In the uniform-base control, the
transformer's loss above the exact uniform prediction is only
$0.00035$--$0.00052$ bits, compared with $0.071$--$0.411$ bits for token
counts.

\subsubsection{Shared-base transfer from earlier contexts}

Select histories satisfying $E_t=\{M_{X_t}(X^t)=0\}$: no continuation of the
current token has previously been observed. Equivalently, $X_t$ first appears
at time $t$. Selection depends only on the observed tokens, so
Equation~\eqref{eq:shared-base-predictive} reduces to
\begin{equation}
 Q_B(b\mid x^t,E_t)=\int Z_L(e_b;q)\,\pi(d\alpha,dL,d\sigma\mid x^t).
 \label{eq:application-unseen-row}
\end{equation}
Other rows and the initial token inform the shared posterior. At positive
depth, the posterior mean row $Z_L(e_b;q)$ is generally different from $q_b$;
with uniform base weights it is exactly $1/64$ at every depth.

With expected loss $h_t$ from~\eqref{eq:causal-expected-loss}, define
\begin{equation}
 \operatorname{sav}_M(t;L)=\mathbb E[6-h_t(Q_M)\mid E_t,\alpha=1,L],
 \qquad Q_{\rm KT}(b\mid x^t)=\frac{N_b(x^t)+1/2}{t+32}.
 \label{eq:application-transfer-gain}
\end{equation}
The saving over uniform measures useful prediction before a continuation has been
seen. The comparison uses $256$ independent shared-base sequences of length
$1025$ per depth $L=0,1,4,8$, with $\alpha=1$. Selection varies with time, and
all differences pair predictions on the same eligible histories.

\begin{table}[tbp]
\centering
\small
\begin{tabular}{rrrccc}
\toprule
$L$ & $t$ & $n$ & Larger, run 1 & Larger, run 2 & Token counts \\
\midrule
0 & 4 & 210 & $0.0018\pm0.0011$ & $0.0023\pm0.0012$ & $0.0067\pm0.0030$ \\
0 & 16 & 136 & $0.0106\pm0.0022$ & $0.0090\pm0.0022$ & $0.0173\pm0.0062$ \\
0 & 64 & 64 & $0.0143\pm0.0026$ & $0.0123\pm0.0018$ & $0.0552\pm0.0046$ \\
\addlinespace
1 & 4 & 214 & $0.0005\pm0.0016$ & $0.0026\pm0.0020$ & $-0.0038\pm0.0038$ \\
1 & 16 & 136 & $0.0027\pm0.0033$ & $-0.0030\pm0.0031$ & $-0.0098\pm0.0108$ \\
1 & 64 & 59 & $0.0057\pm0.0058$ & $0.0003\pm0.0058$ & $0.0497\pm0.0107$ \\
\addlinespace
4 & 4 & 223 & $0.0024\pm0.0021$ & $0.0041\pm0.0022$ & $-0.0059\pm0.0037$ \\
4 & 16 & 145 & $-0.0031\pm0.0042$ & $-0.0059\pm0.0040$ & $0.0137\pm0.0144$ \\
4 & 64 & 66 & $0.0139\pm0.0109$ & $0.0106\pm0.0098$ & $0.0317\pm0.0271$ \\
\addlinespace
8 & 4 & 228 & $0.0035\pm0.0028$ & $-0.0011\pm0.0029$ & $-0.0013\pm0.0051$ \\
8 & 16 & 161 & $0.0021\pm0.0041$ & $0.0076\pm0.0042$ & $0.0721\pm0.0169$ \\
8 & 64 & 57 & $0.0117\pm0.0110$ & $0.0067\pm0.0098$ & $0.2110\pm0.0487$ \\
\bottomrule
\end{tabular}
\caption{\textbf{Loss above informed Bayes after a new token.}
Shared-base source at $\alpha=1$: expected next-token loss difference in bits,
plus or minus one standard error across the $n$ selected sequences. Positive
values favor informed Bayes. Both larger-model runs use the same histories;
Bayesian numerical error is assessed separately.}
\label{tab:application-unseen-paired}
\end{table}

\subsubsection{Uniform-base control}

By symmetry and conditional row independence, an unobserved row has posterior
mean $U_{64}$, also conditional on any fixed generating depth. Hence
\begin{equation}
 \mathbb E[h_t(Q)-h_t(U_{64})\mid X^t,E_t]
 =-\frac1{64}\sum_b\log_2 Q(b\mid X^t)-6
 =D_{\mathrm{KL},2}(U_{64}\Vert Q).
 \label{eq:application-unseen-uniform-kl}
\end{equation}
The two-layer uniform-base transformer of
Section~\ref{sec:bayes-uniform-base} is evaluated on $1024$ sequences
whose depth is uniform on $0,\ldots,8$. At depths $0,1,4,8$, the eligible
counts at times $4/16/64$ are respectively $107/84/52$, $98/80/40$,
$110/81/37$ and $110/84/35$. Every eligible sequence enters. At
$L=1,t=64$, transformer and token-count losses above the uniform prediction are
$0.000365\pm0.000018$ and $0.3369\pm0.0063$ bits.

\subsubsection{Numerical validation}

Time~$4$ uses deterministic rank marginalization; times~$16$ and~$64$ use
independent $16$- and $32$-node calculations with the criteria in
Appendix~\ref{app:bayes-evidence-details}. All twelve selected test sets pass.
At time~$64$, the refinements retain all $66$ and $57$ eligible histories at
$L=4,8$. Their coarse/refined numerical error bounds are
$0.00213/0.00207$ and $0.00243/0.00217$ bits, below the $0.00333$-bit budget;
the corresponding integration-rule shifts are $0.00111$ and $0.00271$ bits.
A separate script recomputes and checks source identities, probabilities and paired
loss statistics.

\subsection{When copying starts to pay}

The final comparison asks when a fixed predictor begins to profit from
copying an earlier continuation of the current token. The shared-base
ordered-versus-shuffled experiment measures this for the transformer:
shuffling preserves token frequencies but destroys the order. Informed Bayes
and the frequency-only Bayesian predictor $B_{\mathrm{freq}}$ give the
corresponding Bayesian prediction. Their next-token gains through sixteen
tokens and at $t=32$ agree in size and depth ordering; the longer transformer curves show how
the gain from copying grows with the number of tokens.

\subsubsection{Behavioral emergence of copying}
\label{app:shared-shuffle-control}

The matched control isolates the context--successor associations used for
copying while preserving every sequence's token counts. The tables compare
both training runs on the same $256$ test sources per depth. The transformer
curves use $16384$ sources per depth: run~1 through $1024$ tokens in
Figure~\ref{fig:applications}(g,h), and both runs through $16$ tokens in
Figure~\ref{fig:shared-shuffle-early}. The Bayesian points through $t=64$
use the paired $256$-source test set. Both training runs
show the same depth ordering. At $t=1024$, run~1 gains
$0.569$ bits per prediction at $L=2$ and $0.880$ bits at $L=3$, with
cumulative gains of about $417$ and $704$ bits; run~2 gives $0.568$ and
$0.883$ bits per prediction and cumulative gains of about $418$ and $705$
bits. The $L=0$ control remains negative in both runs. The matched experiment
therefore measures when learning transitions becomes useful, without
identifying a particular internal circuit.

\paragraph{Matched training and evaluation.}
For each shared-base training sequence $x_1,\ldots,x_{1025}$, the control
keeps $x_1$ and independently permutes positions $2,\ldots,1025$. This
preserves the first token and full count vector exactly. Token frequencies
estimate the realized stationary law $\nu T=\nu$, rather than generally the
latent base $q$; conditional on $x_1$, the shuffled training law is invariant
to the order of the remaining tokens.

Each of two ordered/shuffled model pairs shares its initialization, matched
batches of $64$ source sequences, architecture, optimizer and learning-rate
schedule. Both receive $100000$ updates and $6.4$ million sequences. The
four-layer, width-$256$ models have $3454464$ parameters. Indexed batch
regeneration, independent replay and first-token/count preservation verify
the match. The paired tables use the same $256$ ordered sequences of length $1025$
per condition, with $\alpha=1$ and $L=0,2,3$; neither model receives these
test parameters. Let $h_t$ be the expected loss under the true row,
from~\eqref{eq:causal-expected-loss}. The difference between the
shuffled-trained and the ordered-trained model,
$h_t(\mathrm{TF}_{\mathrm{shuf}})-h_t(\mathrm{TF}_{\mathrm{ord}})$, defines
the transformer gain $g_{\mathrm{TF}}(t)$
of~\eqref{eq:shared-shuffle-payoff}, and its sum within a sequence defines
$G_{\mathrm{TF}}(n)$. Whole-sequence resampling retains all times and pairings.
For the $16384$-source evaluation, the endpoint gains at $L=2,3$ are
$0.5678\pm0.0025$ and $0.9087\pm0.0036$ bits per token, with cumulative
gains of $419.83\pm0.49$ and $700.40\pm0.69$ bits. The $L=0$ control gives
$-0.00410\pm0.00005$ bits per token. Figure~\ref{fig:applications}(g,h)
uses $1000$ paired whole-sequence bootstrap replicates for its pointwise
$95\%$ bands.

\begin{table}[tbp]
\centering\small
\begin{tabular}{ccrr}
\toprule
Depth & Run & Next-token gain at $1024$ & Cumulative gain through $1024$\\
\midrule
$0$ & 1 & $-0.00386\pm0.00034$ & $-9.40\pm0.41$\\
$0$ & 2 & $-0.00323\pm0.00032$ & $-6.59\pm0.33$\\
$2$ & 1 & $0.5690\pm0.0193$ & $416.8\pm3.8$\\
$2$ & 2 & $0.5685\pm0.0189$ & $417.6\pm3.8$\\
$3$ & 1 & $0.8795\pm0.0280$ & $703.5\pm5.3$\\
$3$ & 2 & $0.8827\pm0.0278$ & $704.5\pm5.4$\\
\bottomrule
\end{tabular}
\caption{\textbf{The gain from copying increases with source depth.}
Loss of the shuffled-trained model minus that of the ordered-trained model,
in bits; means $\pm$ paired standard errors across $256$ sequences per depth.
Positive values favor the model that can learn context--successor
associations.}
\label{tab:shared-shuffle-endpoints}
\end{table}

\paragraph{The first tokens.}
Figure~\ref{fig:shared-shuffle-early} shows how copying begins to pay.
At both depths, the mean next-token gain is positive from $t=2$ in both
training runs. Cumulative mean gains become positive within the first seven
observations, and the benefit grows more quickly at $L=3$.

\begin{figure}[tbp]
\centering
\includegraphics[width=\linewidth]{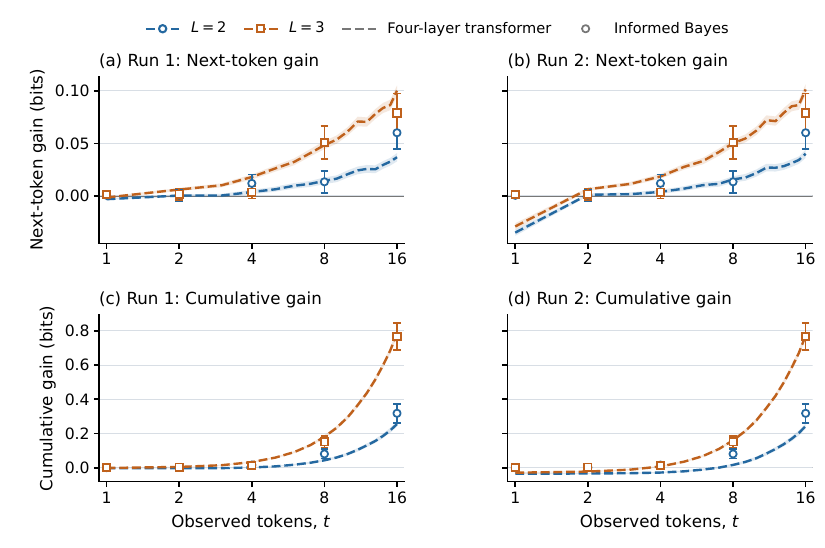}
\caption{\textbf{The early gain from copying increases with source depth.}
The first $16$ tokens at $L=2$ (blue) and $L=3$ (orange), with training
runs in separate columns. Rows show next-token and cumulative gains. Dashed
curves compare the four-layer ordered- and shuffled-trained transformers
on $16384$ sources per depth. Open circles and squares compare informed Bayes with the
frequency-only predictor on $256$ sources per depth. Shading gives pointwise
$95\%$ paired whole-sequence bootstrap intervals; bars give one standard
error for Bayes. Zero marks equal loss.}
\label{fig:shared-shuffle-early}
\end{figure}

\paragraph{The frequency-only Bayesian predictor.}
For comparison we need a Bayesian predictor that, like the shuffled-trained
transformer, can use token frequencies but not token order.
Draw $(q,T)$ from the training prior and let $\nu(T)$ be its
stationary law. The corresponding independent-token surrogate, retaining
the distinguished first-token law, is
\begin{equation}
 m_{\mathrm{freq}}(x^t)=\mathbb E_{q,T}\!\left[q_{x_1}
                  \prod_{s=2}^{t}\nu(T)_{x_s}\right],\qquad
 B_{\mathrm{freq}}(b\mid x^t)=\frac{m_{\mathrm{freq}}(x^t b)}{m_{\mathrm{freq}}(x^t)}.
 \label{eq:stationary-frequency-surrogate}
\end{equation}
This predictor integrates the source prior and receives none of the test
source's $q$, $T$ or $\nu$. It learns frequencies rather than transitions
and approximates the finite shuffled-training law.

At every $t=1,\ldots,16$, its mean expected loss is within $0.0043$ bits of
either shuffled transformer across $L=0,2,3$. After one token it gives a
repeat of that token probability $0.06790$, close to the $0.06491$ of
informed Bayes. The Bayesian order gain $h_t(B_{\mathrm{freq}})-h_t(B)$ and its cumulative sum
are negative throughout at $L=0$ and positive from $t=1$ at $L=2,3$. At
$t=16$, the per-token gains are $0.0601\pm0.0151$ and
$0.0788\pm0.0185$ bits, and the cumulative gains are
$0.3179\pm0.0563$ and $0.7667\pm0.0807$ bits.

\begin{table}[tbp]
\centering\small
\begin{tabular}{ccrrr}
\toprule
Depth & $t$ & Informed Bayes & Transformer run 1 & Transformer run 2\\
\midrule
$0$ & $16$ & $-0.0642\pm0.0077$ & $-0.0854\pm0.0089$ & $-0.0858\pm0.0088$\\
$2$ & $16$ & $0.0601\pm0.0151$ & $0.0537\pm0.0158$ & $0.0599\pm0.0150$\\
$3$ & $16$ & $0.0788\pm0.0185$ & $0.0839\pm0.0190$ & $0.0848\pm0.0182$\\
\midrule
$0$ & $32$ & $-0.0311\pm0.0045$ & $-0.0670\pm0.0066$ & $-0.0607\pm0.0065$\\
$2$ & $32$ & $0.0616\pm0.0163$ & $0.0545\pm0.0172$ & $0.0568\pm0.0165$\\
$3$ & $32$ & $0.1829\pm0.0265$ & $0.1842\pm0.0274$ & $0.1813\pm0.0269$\\
\bottomrule
\end{tabular}
\caption{\textbf{The Bayesian and transformer gains from token order.}
Next-token gain after $16$ and $32$ observations, in bits; means $\pm$ paired
standard errors on the same $256$ test sources per depth. Bayes is compared
with its stationary-frequency predictor, and each transformer with its
matched shuffled-trained model.}
\label{tab:shared-shuffle-bayes-sixteen}
\end{table}

Table~\ref{tab:shared-shuffle-bayes-sixteen} compares the predictors on the
same histories. The surrogate is specified entirely by the source prior,
independently of transformer scores. Its calculation averages alphabet
permutations analytically over $256$ independent source banks, with
$32$- and $16$-node exponent rules and equal weight on depths $0,\ldots,8$.
The informed Bayesian calculation uses $16$- and $32$-node exponent rules
on all $768$ histories at each $t=5,\ldots,16$ and at $t=32$; the first four tokens use
deterministic rank marginalization. Shared canonical references and shared
source banks are accounted for in the numerical uncertainty. All $96$
per-token and cumulative coordinates through $t=16$, and the three
next-token coordinates at $t=32$, pass the sampler,
integration-order and aggregate-precision checks. The largest numerical
allowance is $0.00311394$ bits; each coordinate must stay below both
$0.00333333$ bits and one quarter of its sequence standard error.

The preliminary $t=128$ estimates use the first $32$ test sequences per
depth and the same frequency-only predictor, evaluated with one integration
bank, $32$ exponent nodes, and one draw at each positive depth. The
depth-zero contribution and equal prior weights on $L=0,\ldots,8$ are retained.

\section{Summary of transformer--Bayes gaps}
\label{app:gap-table}

Table~\ref{tab:gap-summary} summarizes the gaps of Section~\ref{sec:bayes-evidence} across models, settings and training runs; the underlying comparisons are in Appendix~\ref{app:bayes-evidence-details}.

\begin{table}[tbp]
\caption{\textbf{Gap between transformer and informed Bayes, in bits per
token.} ``Largest gap'' ranges over the displayed settings; ``last gap''
uses their largest $t$. Pooled rows summarize $t\le4$. Two
values separated by a slash are the two training runs; ranges include all
listed settings and runs.}
\label{tab:gap-summary}
\centering\small
\setlength{\tabcolsep}{3.5pt}
\begin{tabular}{lccll}
\toprule
Source & Layers & Runs & Largest gap & Last gap\\
\midrule
Known encoding & 2 & 1 & $0.124$ ($t=32$) & $0.004$--$0.009$ ($t=1024$)\\
Known encoding & 4 & 2 & $0.044$ / $0.049$ ($t=16$) & $0.0012$--$0.0056$ ($t=1024$)\\
Hidden encoding & 2 & 2 & $0.13$ / $0.13$ ($t=256$) & $0.105$ / $0.096$ ($t=1024$)\\
Hidden encoding & 8 & 2 & $0.025$ / $0.022$ ($t=256$) & $0.0024$ / $0.0018$ ($t=1024$)\\
Uniform base & 2 & 1 & $0.087$ ($L=0$, $t=128$) & $0.024$--$0.068$ ($t=1024$)\\
Uniform base & 4 & 2 & $0.037$ / $0.041$ ($L=8$, $t=64$) & $0.00003$--$0.028$ ($t=1024$)\\
Shared base & 2 & 2 & $0.217$ / $0.220$ ($L=0$, $t=64$) & ---\\
Shared base & 8 & 2 & $0.0066$ / $0.0052$ ($t\le64$) & $-0.0010$--$0.0051$ ($t\ge256$)\\
Pooled families & 2 & 2 & $0.0197$ / $0.0190$ ($t=4$) & $0.0197$ / $0.0190$ ($t=4$)\\
Pooled families & 4 & 2 & $0.0080$ / $0.0067$ ($t=1$) & $0.0052$ / $0.0025$ ($t=4$)\\
\bottomrule
\end{tabular}
\end{table}

\end{document}